\documentclass{article} 
\usepackage{iclr2025_conference,times}

\usepackage{amsmath,amsfonts,bm}

\def\eqref#1{equation~\ref{#1}}

\def\1{\bm{1}}

\DeclareMathAlphabet{\mathsfit}{\encodingdefault}{\sfdefault}{m}{sl}
\SetMathAlphabet{\mathsfit}{bold}{\encodingdefault}{\sfdefault}{bx}{n}

\DeclareMathOperator*{\argmin}{arg\,min}

\usepackage{hyperref}
\usepackage{url}
\usepackage{mathtools}
\usepackage{mathrsfs}
\usepackage{algorithm}
\usepackage[noend]{algorithmic}
\usepackage[makeroom]{cancel}
\usepackage{wrapfig}
\usepackage[export]{adjustbox}
\usepackage{bbm}
\usepackage{booktabs}
\usepackage{multicol}
\usepackage{multirow}
\usepackage{makecell}

\usepackage{tikz}
\usetikzlibrary{tikzmark}

\usepackage[capitalize,noabbrev]{cleveref}

\crefname{defn}{Definition}{Definition}
\crefname{section}{Section}{Section}
\crefname{algorithm}{Algorithm}{Algorithm} 
\crefname{thm}{Theorem}{Theorem}
\crefname{lem}{Lemma}{Lemma}
\crefname{prop}{Proposition}{Proposition}
\crefname{asm}{Asm.}{Asm.}
\crefname{appendix}{Appendix}{Appx.}
\crefname{equation}{Equation}{Equations}
\crefname{figure}{Figure}{Figure}
\crefname{table}{Table}{Table}
\crefname{cor}{Corollary}{Corollary}

\usepackage{notation}
\usepackage{researchpack}

\title{Discrete Copula Diffusion}

\author{Anji Liu$^{1,2}$,~Oliver Broadrick$^{1}$, ~Mathias Niepert$^{2}$,~{Guy Van den Broeck}$^{1}$
\\\\
$^{1}$Department of Computer Science, University of California, Los Angeles, USA
\\
$^{2}$Institute for Artificial Intelligence, University of Stuttgart, Germany \\
\texttt{\{liuanji,obroadrick,guyvdb\}@cs.ucla.edu}
\\
\texttt{mathias.niepert@ki.uni-stuttgart.de}
}

\iclrfinalcopy 
\begin{document}

\maketitle

\begin{abstract}
Discrete diffusion models have recently shown significant progress in modeling complex data, such as natural languages and DNA sequences. However, unlike diffusion models for continuous data, which can generate high-quality samples in just a few denoising steps, modern discrete diffusion models still require hundreds or even thousands of denoising steps to perform well. In this paper, we identify a fundamental limitation that prevents discrete diffusion models from achieving strong performance with fewer steps -- they fail to capture dependencies between output variables at each denoising step. To address this issue, we provide a formal explanation and introduce a general approach to supplement the missing dependency information by incorporating another deep generative model, termed the \emph{copula} model. Our method does not require fine-tuning either the diffusion model or the copula model, yet it enables high-quality sample generation with significantly fewer denoising steps. When we apply this approach to autoregressive copula models, the combined model outperforms both models individually in unconditional and conditional text generation. Specifically, the hybrid model achieves better (un)conditional text generation using 8 to 32 times fewer denoising steps than the diffusion model alone. In addition to presenting an effective discrete diffusion generation algorithm, this paper emphasizes the importance of modeling inter-variable dependencies in discrete diffusion.\footnote{Code is available at \href{https://github.com/liuanji/Copula-Diffusion}{\color{blue}{https://github.com/liuanji/Copula-Diffusion}}.}
\end{abstract}

\section{Introduction}


Discrete diffusion models have recently achieved significant progress in modeling complex data such as natural languages \citep{campbell2022continuous,sahoo2024simple}, protein sequences \citep{gruver2023protein,morehead2023towards}, and graphs \citep{vignac2022digress,huang2023conditional}. In particular, recent discrete diffusion models for text generation \citep{lou2024discrete,sahoo2024simple,shi2024simplified} have matched or even surpassed the performance of autoregressive models at the scale of GPT-2 \citep{radford2019language}. Additionally, discrete diffusion models offer improved inference-time controllability using guidance from auxiliary models such as classifiers \citep{dhariwal2021diffusion}, making them suitable for controlled generation tasks \citep{li2022diffusion,han2023ssd}.

Despite these promising results, discrete diffusion models still require hundreds to thousands of denoising steps to produce high-quality samples \citep{austin2021structured,sahoo2024simple}, significantly affecting their efficiency. In this paper, we identify a fundamental limitation in most discrete diffusion models that hinders their ability to generate high-quality samples in just a few steps.

We illustrate the problem in \cref{fig:pipeline}. At each denoising step, a partially completed sample shown in the top-left is fed into a sequence-to-sequence denoising model, which predicts the univariate marginal distributions for each masked token independently. A new output sequence is then sampled based on these univariate marginals before proceeding to the next denoising step. The key issue with this process is that when multiple ``edits'' (\ie replacing masked tokens with data tokens) are made simultaneously, the model does not account for the joint probability of these changes occurring together. As a result, the generated samples often lack coherence, as shown in the bottom-left of \cref{fig:pipeline}. This problem is exacerbated in few-step generation, where many tokens must be edited simultaneously. We formally demonstrate that if the diffusion model predicts each variable independently, an irreducible term (in addition to the data entropy) remains in the negative evidence lower bound (ELBO), preventing the model from perfectly capturing the data distribution.

We propose using a generative model, which we refer to as the copula model, to compensate for the missing dependency information between output variables at each denoising step. Our method operates only at inference time and can be adapted to any discrete diffusion model and a wide range of copula models. 
As illustrated on the right side of \cref{fig:pipeline}, the input sequence is also fed into a copula model that (implicitly) produces information on inter-variable dependencies. This information is combined with the univariate marginals predicted by the diffusion model to produce a more accurate distribution, resulting in higher-quality samples, shown in the bottom-right~corner.

We formally show that the univariate marginals from the diffusion model and the dependencies captured by the copula model can be combined in a principled way, leading to a better approximation of the true denoising distribution under mild assumptions. Further, finding this combined distribution reduces to solving a convex optimization problem that can be efficiently approximated in practice.

By instantiating the copula model as an autoregressive deep generative model such as GPT \citep{radford2019language}, we propose an algorithm that combines any pretrained discrete diffusion model with an autoregressive model to form a hybrid model called \textbf{D}iscrete \textbf{C}opula \textbf{D}iffusion (DCD). This model is capable of producing high-quality (un)conditional samples with only a few denoising steps. Empirical results on text and antibody generation show that DCD significantly outperforms both of its base models. Moreover, DCD achieves comparable or better performance using 8 to 32 times fewer denoising steps compared to the base discrete diffusion model. In addition to proposing a discrete diffusion model capable of few-step generation, we emphasize the importance of modeling inter-variable dependencies in discrete diffusion models.

\begin{figure}[t]
    \centering
    \includegraphics[width=\linewidth]{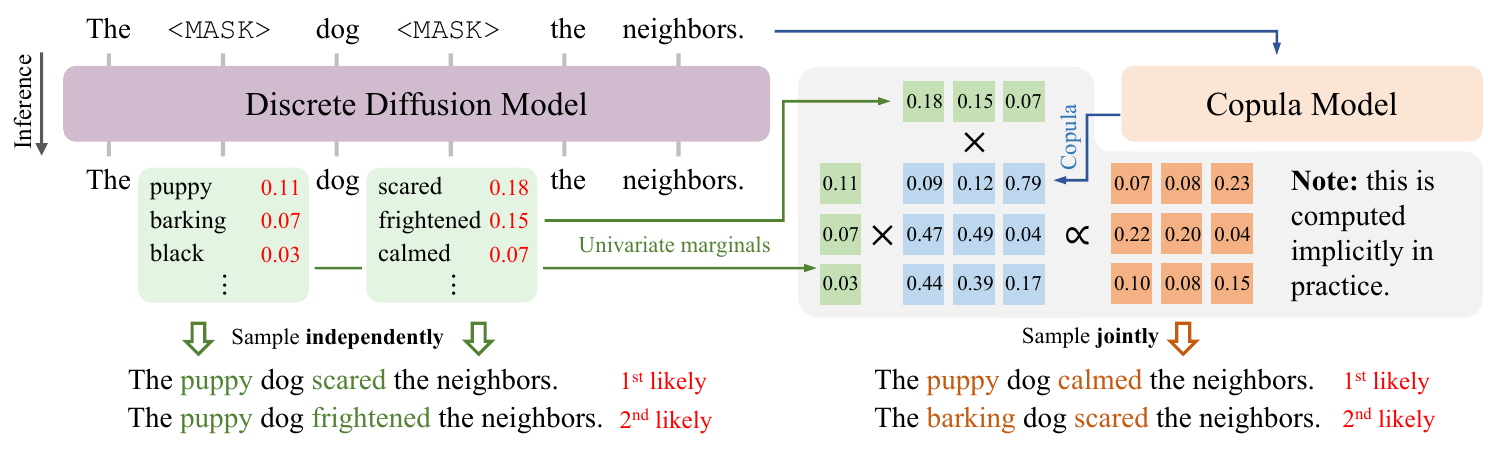}
    \vspace{-2.4em}
    \caption{\textbf{Discrete Copula Diffusion (DCD).} At each denoising step, a partially completed sequence is given as input (top-left). The diffusion model independently predicts the univariate marginals for each masked token, which leads to the samples in the bottom-left. DCD introduces an additional copula model (top-right) to capture the inter-variable dependencies, thereby supplementing the information missed by the diffusion model. By combining outputs from both models in a principled way, DCD achieves better performance than either model individually (see improved samples in the bottom-right), enabling few-step discrete diffusion generation.}
    \label{fig:pipeline}
    \vspace{-0.6em}
\end{figure}

\section{Preliminaries}
\label{sec:preliminaries}

We aim to model the joint distribution of variables $\X_0$, a set of categorical variables with $C$ categories. Discrete diffusion models \citep{austin2021structured} learn to sample from $\p (\X_0)$ by modeling the reversal of the following noising process involving $\X_0$ and a set of auxiliary variables $\{\X_{t}\}_{t=1}^{T}$:
    \begin{align}
        \forall t \in \{1, \dots, T\} \quad \q(\x_t \given \x_{t-1}) := \cat(\x_t ; Q_t \!\cdot\! \x_{t-1}),
    \end{align}
\noindent where $\cat (\x; \mathbf{p})$ refers to the Categorical distribution over $\x$ with class probabilities $\mathbf{p}$, and $Q_t$ is a $C \!\times\! C$ transition matrix that is applied independently to every variable $x_{t-1}^{i}$ (denote $x^{i}_{t-1}$ as the $i$th variable of $\x_{t-1}$) to get the corresponding categorical distribution of $x_{t}^{i}$. Specifically, each variable $x_{t-1}^{i}$ is treated as a one-hot vector of size $C \!\times\! 1$, which is then multiplied by $Q_t$ to compute the class probabilities of $x_{t}^{i}$. The noising process is designed such that $\p(\x_{T})$ follows a simple distribution regardless of the data distribution. 

Instead of using a fixed number of predefined time steps, we can treat $t$ as a continuous variable within the range $[0, T]$. The noising process is now defined by the rate of change of $\p(\x_{t})$ \wrt $t$: $\frac{d \p(\x_{t})}{d t} \!=\! Q \!\cdot\! \p(\x_{t})$, where $Q \!\in\! \mathbb{R}^{C \times C}$ is a transition rate matrix. For any $0 \!\leq\! s \!<\! t \!\leq\! T$, we have 
    \begin{align*}
        \q (\x_{t} \given \x_{s}) := \cat (\x_{t}; \exp((t\!-\!s) \!\cdot\! Q) \!\cdot\! \x_{s}), 
    \end{align*}
\noindent where $\exp(\cdot)$ denotes the matrix exponential.

Discrete diffusion models represent the reverse diffusion process as a Markov chain from $\x_T$ to $\x_0$, effectively reversing the noising process. Specifically, the reverse diffusion is modeled as: 
    \begin{align*}
        \p_{\theta} (\x_{0:T}) := \p(\x_T) \prod_{t=0}^{T-1} \p_{\theta} (\x_{t} \given \x_{t+1}).
    \end{align*}
In the discrete-time framework, the model is trained by maximizing the ELBO, which is defined by the forward joint distribution ($\q(\x_{1:T} \given \x_{0}) \p(\x_0)$) and the reverse joint distribution ($\p_{\theta} (\x_{0:T})$) \citep{ho2020denoising}. In the continuous-time framework, we can either adopt an extended ELBO objective \citep{zhao2024improving} or to learn the likelihood ratios $\{\p(\x'_{t}) / \p(\x_{t})\}_{\x_t, \x'_t}$, allowing for the recovery of $\p(\x_{s} \given \x_{t})$ ($s \!<\! t$) in an indirect manner \citep{lou2024discrete,meng2022concrete,sun2022score}. Following the reverse diffusion process, sampling from a diffusion model involves first sampling from the prior $\p(\x_T)$ and then recursively sampling $\x_{T-1}, \dots, \x_{0}$ following $\{\p_{\theta} (\x_{t} \given \x_{t-1})\}_{t=0}^{T-1}$.


\section{Challenge of Modeling Variable Dependencies}
\label{sec:challenges}

Unlike continuous diffusion models, which can produce high-quality samples with just a few steps (\eg \citet{song2023consistency,zhou2024score}), discrete diffusion models exhibit a strong positive correlation between sample quality and the number of denoising steps. For instance, to generate $1024$ text tokens, a recent discrete diffusion model SEDD \citep{lou2024discrete} requires $1024$ steps to reach around $35$ perplexity (PPL), while with $32$ denoising steps the PPL is only around $130$.

We argue that the need for a large number of sampling steps in discrete diffusion models stems from their inability to capture inter-variable dependence among the outputs. Specifically, at each time step $t$, discrete diffusion models independently sample each variable from $\x_{t}$ conditioned on $\x_{t+1}$, \ie $\p(\x_{t} \given \x_{t+1}) \!:=\! \prod_{i} \p(x_{t}^{i} \given \x_{t+1})$. As a result, when changing multiple variables from $\x_{t+1}$ to $\x_{t}$, the model fails to account for the joint probability of these modifications happening together.
In the following, we first quantitatively analyze the performance degradation caused by this independent denoising assumption. We then discuss approaches to mitigate this issue.


\boldparagraph{Quantifying the Performance Drop.} 
The total correlation of a distribution $\p(\X)$ is the KL-divergence between itself and the product of its univariate marginals:
    {\setlength{\belowdisplayskip}{-0.2em}
    \begin{align*}
        \tc (\p(\X)) \!:=\! \sum_{\x} \p(\x) \log \Big ( \p(\x) / \prod_{i} \p(x_i) \Big ).
    \end{align*}}

The following result demonstrates that, under the independent denoising assumption, there is an irreducible component in the ELBO that directly stems from ignoring inter-variable dependencies.

\begin{prop}
\label{prop:elbo-decomp}
    Assume the denoising distributions $\{\p_{\theta} (\x_{t} \given \x_{t+1})\}_{t=0}^{T-1}$ are fully factorized. Let $\mathrm{H} (\p(\X))$ denote the entropy of $\p(\X)$. For any choice of denoising distributions (or equivalently, any parameterization $\theta$), the negative ELBO of the diffusion model is lower bounded by
        {\setlength{\belowdisplayskip}{0.0em}
        \begin{align}
            \mathrm{H} (\p(\X_0)) + \sum_{t=1}^{T} \tc (q (\X_{t-1} \given \X_{t})), \text{~~~~where~} \tc (\p(\Y \given \X)) := \expectation_{\x \sim \p} \big [ \tc (\p(\Y \given \x)) \big ].
            \label{eq:elbo-bound}
        \end{align}}
    \vspace{-0.4em}
\end{prop}

The first term represents the entropy of the data distribution and is irreducible. The second term additionally depends on the noising process and the chosen noise levels, which set an upper limit on the performance of discrete diffusion models that use the independent denoising assumption. Note that although $\tc (\q (\X_{t} \given \X_{t-1}))$ is zero according to the definition of the noising process, $\tc (\q (\X_{t-1} \given \X_{t}))$ is not unless the data distribution is fully factorized.

\boldparagraph{Closing the Performance Gap.} 
While increasing the number of denoising steps can improve sample quality, it also introduces significant computational overhead during inference. Our goal is to use fewer denoising steps while maintaining good sample quality. As shown in \cref{prop:elbo-decomp}, given a fixed noising strategy and the number of denoising steps, the only way to reduce the negative ELBO lower bound in \cref{eq:elbo-bound} is to relax the independent denoising assumption. That is, in addition to modeling the univariate marginals, we must also account for dependencies between variables.

The challenge of capturing inter-variable dependencies during each denoising step can be addressed through adjustments during either training or inference. A direct approach involves modeling both the univariate marginals and the inter-variable dependencies within the diffusion model. However, this requires improving existing sequence-to-sequence architectures (\eg \citet{devlin2018bert}) to capture dependencies between output variables directly, which is not very well studied in the literature.



Instead, we propose an inference-time solution that complements the information missed by the pretrained discrete diffusion model. Specifically, we aim to combine the univariate marginals produced by the diffusion model with the inter-variable dependencies learned by another (possibly smaller) deep generative model, which we refer to as the \emph{copula model}. The term ``copula'' traditionally refers to the dependencies between random variables in statistics \citep{nelsen2006introduction}.




\section{Modeling Variable Dependencies with Copula Models}

As motivated in the previous section, our main goal is to combine the univariate marginals produced by the diffusion model with the inter-variable dependencies captured by a copula model. In this section, we first formalize the concept of ``combining'' two such distributions in a general context (Sec.~\ref{sec:copula-general}). We then specialize the formulation to the case of diffusion models (Sec.~\ref{sec:copula-dm}).

\subsection{Combining Univariate Marginals with Inter-Variable Dependencies}
\label{sec:copula-general}

In this section, we discuss how to best inject inter-variable dependence using copula models given a target distribution $\p_{\mathrm{tar}}$ over $\X$. Assume we have access to $\p_{\mathrm{tar}}$ through two sources: (i) the set of all univariate marginal distributions $\{\p_{\mathrm{tar}}(X_{i})\}_{i}$, and (ii) an estimate $\p_{\mathrm{est}}$ of the target distribution coming from the copula model, which is also a generative model. Our goal is to combine these two estimates to construct $\hat{\p}$ that is ``closer'' to the true distribution $\p_{\mathrm{tar}}$ than either estimate individually.

We construct $\hat{\p}$ as the distribution that (i) matches the set of univariate marginals $\{\p_{\mathrm{tar}}(X_{i})\}_{i}$, and (ii) minimizes the KL divergence to $\p_{\mathrm{est}}$. The intuition is that by ensuring $\hat{\p}$ has the correct univariate marginals, we can achieve a good approximation of $\p_{\mathrm{tar}}$ even if $\p_{\mathrm{est}}$ is biased. To formalize this, we first define information projection (I-projection).

\begin{defn}
\label{defn:I-projection}
    The I-projection of a distribution $q(\X)$ onto a set of distributions $\calP$ over $\X$ is
        {\setlength{\belowdisplayskip}{0.0em}
        \begin{align*}
            \p^{*} = \argmin_{\p \in \calP} \kld (\p \;\|\; \q).
        \end{align*}}
    \vspace{-0.4em}
\end{defn}

Let $\calP^{\p}_{\mathrm{mar}}$ denote the set of distributions over $\X$ that share the same univariate marginals as $\p$. We define $\hat{\p}$ as the I-projection of $\p_{\mathrm{est}}$ onto $\calP^{\p_{\mathrm{tar}}}_{\mathrm{mar}}$. The following proposition shows that regardless of the initial estimate $\p_{\mathrm{est}}$ of $\p_{\mathrm{tar}}$, the I-projection $\hat{\p}$ will be an improved estimate of $\p_{\mathrm{tar}}$ in KL-divergence.


\begin{prop}
\label{prop:iproj-is-good}
If there exists $i$ and $x_{i}$ s.t. $\p_{\mathrm{tar}} (x_{i}) \!\neq\! \p_{\mathrm{est}} (x_{i})$, then $\kld (\p_{\mathrm{tar}} \|\; \hat{\p}) \!<\! \kld (\p_{\mathrm{tar}} \|\; \p_{\mathrm{est}}).$
\end{prop}



Having now seen that $\hat{\p}$ is an improved estimate of $\p_{\mathrm{tar}}$, we next explore whether it is feasible to compute $\hat{\p}$ given $\{\p_{\mathrm{tar}} (X_{i})\}_{i}$ and $\p_{\mathrm{est}}$. We start by showing that $\hat{\p}$ has a simple form.

\begin{prop}
\label{prop:phat-form}
    Assume $\forall \x$, $\p_{\mathrm{tar}} (\x) \!>\! 0$ and $\p_{\mathrm{est}} (\x) \!>\! 0$. Then $\hat{\p}$ exists and has the form
        {
        \setlength{\abovedisplayskip}{0.4em}
        \setlength{\belowdisplayskip}{-0.4em}
        \begin{align*}
            \hat{\p} (\x) = \p_{\mathrm{est}} (\x) \cdot \prod_{i} \sigma_{i} (x_{i}), 
        \end{align*}}
        
    \noindent where $\sigma_{i}$ is a positive function 
    that depends on $x_i$.
    \vspace{-0.2em}
\end{prop}

Assume $\X$ consists of $N$ categorical variables, each with $C$ categories, we can represent the factors $\{\sigma_{i}\}_{i}$ using a matrix $\V \!\in\! \mathbb{R}^{N \times C}$. Under this representation, the combined distribution is 
    {
    \setlength{\abovedisplayskip}{0.4em}
    \setlength{\belowdisplayskip}{-0.4em}
    \begin{align}
        \hat{\p} (\x) = \p_{\mathrm{est}} (\x) \cdot \prod_{i} \exp (\V [i, x_i]), \label{eq:phat-form}
    \end{align}}

\noindent where $\V[i, j]$ denotes the element at the $i$th row and $j$th column of $\V$ and $\V [i, x_i] \!=\! \log \sigma_{i} (x_{i})$. Determining the true matrix $\V^{*}$ corresponding to $\hat{\p}$, which is the I-projection of $\p_{\mathrm{est}}$ onto $\calP_{\mathrm{mar}}^{\p_{\mathrm{tar}}}$, can be reformulated as solving the following convex optimization problem.

\begin{thm}
\label{thm:copula-obj}
    If $\V^{*}$ minimizes the following convex objective function, then the corresponding $\hat{\p}$ defined by \cref{eq:phat-form} is the I-projection of $\p_{\mathrm{est}}$ onto $\calP_{\mathrm{mar}}^{\p_{\mathrm{tar}}}$.\footnote{\cref{eq:copula-obj} closely resembles the matrix scaling problem \citep{idel2016review}. See \cref{appx:relation} for details.}
        {\setlength{\abovedisplayskip}{0.4em}
        \setlength{\belowdisplayskip}{-0.4em}
        \begin{align}
            \calL (\V; \p_{\mathrm{tar}}, \p_{\mathrm{est}}) := \sum_{\x} \p_{\mathrm{est}} (\x) \cdot \prod_{i} \exp (\V [i, x_i]) - \sum_{i=1}^{N} \sum_{x_i=1}^{C} \V [i, x_i] \cdot \p_{\mathrm{tar}} (x_i).
            \label{eq:copula-obj}
        \end{align}}
    \vspace{-0.4em}
\end{thm}


We proceed to explain why I-projecting $\p_{\mathrm{est}}$ leads to a better estimate of $\p_{\mathrm{tar}}$ as suggested by \cref{prop:iproj-is-good}. In general, a joint distribution can be viewed as combining two independent pieces of information: (i) a set of univariate marginal distributions and (ii) a \textit{copula} describing the association or dependence among the variables. By the classical work of \citet{sklar1959fonctions}, for continuous variables the copula can take the form of a joint distribution with uniform margins and can be combined quite simply with univariate marginal distributions to recover the full joint distribution, a fact heavily exploited in statistics \citep{nelsen2006introduction}. 
While the discrete case is somewhat less straightforward, recent work of \citet{geenens2020copula} has developed the fundamental notions of discrete copula modeling as well, where the information of a copula can be parameterized by odds ratios.

\cref{fig:copula-example} shows an example consisting of two binary variables $X$ and $Y$. The probability table on the left can be equivalently expressed using univariate marginals (\ie $\p_{0 \cdot}, \p_{1 \cdot}$, $\p_{\cdot 0}$, $\p_{\cdot 1}$) and the odds ratio (\ie copula) $\omega \!:=\! \frac{\p_{00} \p_{11}}{\p_{01} \p_{10}}$ as shown in the middle of \cref{fig:copula-example}. Intuitively, $\omega \!=\! 125$ indicates that the phrases ``alpine skiing'' and ``scuba diving'' are more likely than others (\eg ``alpine diving''), and the marginals decide which of the two phrases appears more frequently. The idea of representing the copula with odds ratios generalizes to the multivariate case and is presented in \cref{appx:discrete-copula}.

The following result demonstrates that, under its functional form in \cref{eq:phat-form}, I-projecting $\p_{\mathrm{est}}$ onto $\calP_{\mathrm{mar}}^{\p_{\mathrm{tar}}}$ only improves the univariate marginals and leaves the copula unchanged regardless of $\V$.

\begin{prop}\label{prop:copula_is_invariant}
For a positive distribution $p$ and any $\V \!\in\! \mathbb{R}^{N\times C}$, the distribution $q(\bm{x}) \!\propto\! \p (\x) \!\cdot\! \prod_{i} \exp (\V [i, x_i])$ has the same copula as $p$.
\vspace{-0.2em}
\end{prop}

In general, \cref{prop:copula_is_invariant} holds because scaling factors (e.g., $\exp (\V [i, x_i])$) cancel in odds ratios. For example, in the $2 \!\times\! 2$ case in \cref{fig:copula-example}, scaling the top row of the probability table by $a$ would result in the odds ratio 
$\omega \!=\! \frac{a\p_{00} \p_{11}}{a\p_{01} \p_{10}} \!=\! \frac{\p_{00} \p_{11}}{\p_{01} \p_{10}}$.


\begin{figure}[t]
    \centering
    \includegraphics[width=\columnwidth]{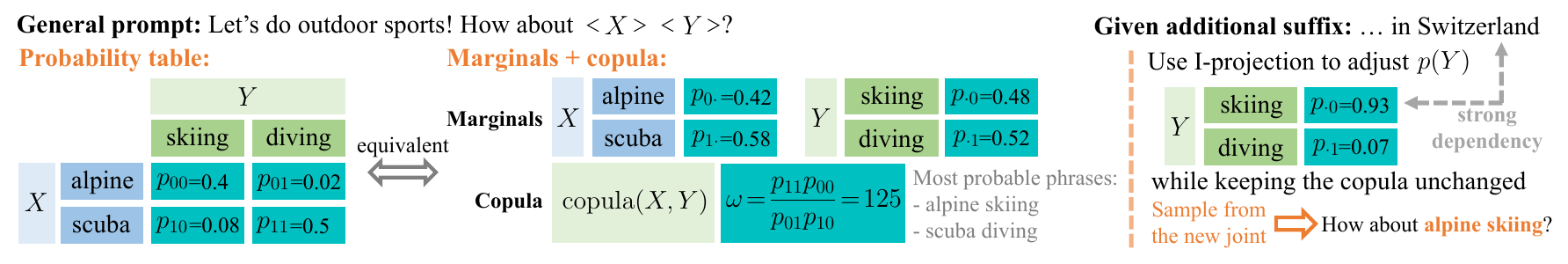}
    \vspace{-2.4em}
    \caption{Illustration of the decomposition of a distribution into univariate marginals and a copula.}
    \label{fig:copula-example}
    \vspace{-1.0em}
\end{figure}

\subsection{Modeling Dependence in Discrete Diffusion Models}
\label{sec:copula-dm}

Recall from \cref{sec:challenges} that our goal is to capture inter-variable dependencies between the output variables at each denoising step (\eg sampling $\x_{t}$ from $\q(\X_{t} \given \x_{t+1})$). Similar to the general case shown in \cref{sec:copula-general}, we first have a set of univariate marginals $\{\p_{\mathrm{dm}} (X_{t}^{i} \given \x_{t+1})\}_{i}$ from the diffusion model. Notably, these univariate marginals are fairly accurate since for both discrete-time and continuous-time diffusion models, if their respective training losses are minimized, the model recovers the true univariate marginals. This is formally justified in \cref{appx:dm-unbiased}.

Alongside the univariate marginals, we assume access to a copula model that encodes a distribution over $\X_{t}$. Following \cref{sec:copula-general}, combining the copula model's distribution with the univariate marginals from the diffusion model will lead to an improved estimate of $\q (\X_{t} \given \x_{t+1})$ (Prop.~\ref{prop:iproj-is-good}).


The performance of the augmented diffusion model hinges on two key questions: (i) how well can the copula model capture the inter-variable dependencies in $\q (\X_{t} \given \x_{t+1})$ (defined by the data distribution and the noising process); (ii) given a good copula distribution, how to effectively combine it with the univariate marginals obtained from the diffusion model, \ie how to solve \cref{eq:copula-obj}.

\section{Autoregressive Models as Copula Models}
\label{sec:ar-as-copula}

This section answers the two questions above tailored to the case where the copula model is an autoregressive model such as GPT \citep{radford2019language} and State Space Models \citep{dao2024transformers}. Specifically, \cref{sec:extract-copula-ar} discusses how to approximate $\q (\X_{t} \given \x_{t+1})$ using an autoregressive model trained on the clean data distribution $\p (\X_{0})$ under certain noising processes. \cref{sec:Iproj-ar} explores the process of performing I-projection from the (autoregressive) copula distribution onto the set of distributions with univariate marginals $\{\p_{\mathrm{dm}} (X_{t}^{i} \given \x_{t+1})\}_{i}$. Finally, \cref{sec:overall-sampling} summarizes the sampling procedure with a discrete diffusion model and an autoregressive copula model.

\subsection{Extracting Copula Distributions from Autoregressive Models}
\label{sec:extract-copula-ar}

At step $t$, to sample $\x_{t}$ conditioned on $\x_{t+1}$, we need a copula distribution $\p_{\mathrm{copula}} (\X_{t})$ that closely approximates $\q (\X_{t} \given \x_{t+1})$. While this might suggest that the copula model should also be trained with a diffusion model objective, which brings us back to the problem of modeling inter-variable dependencies, we show that any model trained on the clean data distribution can serve as a copula model that indirectly approximates $\q (\X_{t} \given \x_{t+1})$ under the absorbing mask forward noising process.

The absorbing mask noising process gradually converts data tokens in $\x_{0} \!\sim\! \p (\X_{0})$ to a new category denoted \texttt{<MASK>} through the sequence $\x_{1}, \dots, \x_{T}$. Specifically, each token in $\x_{0}$ is independently converted to \texttt{<MASK>} with probabilities $0 \!<\! \alpha_{1} \!<\! \dots \!<\! \alpha_{T} \!=\! 1$ in $\x_{1}, \dots, \x_{T}$, respectively. This is a widely used noising strategy for discrete diffusion models. Since this process only transforms data tokens into the mask token, it preserves the dependencies between the remaining unmasked tokens. Therefore, we can decompose $\q (\X_{t} \given \x_{t+1})$ as $\q (\x_{t} \given \x_{t+1}) \!=\! \sum_{\tilde{\x}_{t}} \q (\tilde{\x}_{t} \given \x_{t+1}) \q (\x_{t} \given \tilde{\x}_{t}, \x_{t+1})$, where $\q (\tilde{\x}_{t} \given \x_{t+1})$ is inuitively capturing the joint distribution of generating all currently masked tokens, and $\q (\x_{t} \given \tilde{\x}_{t}, \x_{t+1})$ captures only the choice of which currently masked tokens will actually be generated. Formally, define $I$ as the set of variables $i$ such that $x_{t+1}^{i} \!=\! \text{\texttt{<MASK>}}$ and $J$ as its complement. The auxiliary distributions have the following form.

\begin{prop}
\label{prop:mask-can-be-separated}
    Assume $\p (\X_{0})$ is the clean data distribution and $\{\q (\X_{t} \given \x_{t-1})\}_{t=1}^{T}$ follows the absorbing mask noising process. Let $\alpha_{t}$ be the probability of conversion to the mask state from $X_{0}^{i}$ to $X_{t}^{i}$ ($\forall i$). Define $\tilde{\X}_{t}$ as a set of auxiliary variables such that
        {\setlength{\abovedisplayskip}{0.4em}
        \setlength{\belowdisplayskip}{0.0em}
        \begin{align}
            \q (\tilde{\x}_{t} \given \x_{t+1}) = \p (\X_{0}^{I} = \tilde{\x}_{t}^{I} \given \X_{0}^{J} = \x_{t+1}^{J}) \cdot \mathbbm{1} [ \tilde{\x}_{t}^{J} = \x_{t+1}^{J} ].\label{eq:aux-denoising-dist}
        \end{align}}

    Then, the distribution $\q(\X_{t} \given \tilde{\x}_{t}, \x_{t+1})$ is the following: $\q(\X_{t} \given \tilde{\x}_{t}, \x_{t+1}) \!=\! \prod_{i} \q (x_{t}^{i} \given \tilde{x}_{t}^{i}, x_{t+1}^{i})$. 

    \vspace{-0.2em}
    
    -- For $i \!\in\! I$, $\q (x_{t}^{i} \given \tilde{x}_{t}^{i}, x_{t+1}^{i})$ equals $\alpha_{t} / \alpha_{t+1}$ if $x_{t}^{i} \!=\! \text{\texttt{<MASK>}}$ and equals $1 \!-\! \alpha_{t} / \alpha_{t+1}$ if $x_{t}^{i} \!=\! \tilde{x}_{t}^{i}$. 

    \vspace{-0.2em}
    
    -- For $i \!\in\! J$, $\q (x_{t}^{i} \given \tilde{x}_{t}^{i}, x_{t+1}^{i}) \!=\! 1$ if and only if $x_{t}^{i} \!=\! x_{t+1}^{i}$.
    \vspace{-0.2em}
\end{prop}

Since $\q (\X_{t} \given \tilde{\x}_{t}, \x_{t+1})$ is fully factorized, the copula model only needs to account for inter-variable dependencies in $\q (\tilde{\X}_{t} \given \x_{t+1})$. Following \cref{eq:aux-denoising-dist}, we can transform $\p_{\mathrm{copula}} (\X_{0})$, which estimates the clean data distribution, into $\p_{\mathrm{copula}} (\tilde{\X}_{t} \given \x_{t+1})$ that approximates $\q (\tilde{\X}_{t} \given \x_{t+1})$ by conditioning it on the unmasked tokens in $\x_{t+1}$ (\ie $\x_{t+1}^{J}$). Specifically, for autoregressive copula models (\ie $\p_{\mathrm{copula}} (\x) \!:=\! \prod_{i} \p_{\mathrm{copula}} (x_{i} \given \x_{<i})$), we construct $\p_{\mathrm{copula}} (\tilde{\X}_{t} \given \x_{t+1})$ by conditioning each variable on the corresponding preceding tokens in $\x_{t+1}^{J}$ while enforcing $\tilde{x}_{t}^{j} \!=\! x_{t+1}^{j}$ ($\forall j \!\in\! J$):
    {
    \setlength{\belowdisplayskip}{-0.2em}
    \begin{align}
        \p_{\mathrm{copula}} (\tilde{\x}_{t} \given \x_{t+1}) := \prod_{i \in I} \p_{\mathrm{copula}} (X_{0}^{i} = \tilde{x}_{t}^{i} \given \X_{0}^{<i} = \tilde{\x}_{t}^{<i}) \cdot \prod_{j \in J} \mathbbm{1} [ \tilde{x}_{t}^{j} = x_{t+1}^{j} ]. \label{eq:ar-copula-dgm}
    \end{align}}

This copula distribution is biased even if the autoregressive model perfectly captures the data distribution since it cannot condition on subsequent unmasked tokens in $\x_{t+1}$. In contrast, while being able to condition on all unmasked tokens, diffusion models cannot capture dependence between variables. Combining the two estimates in a proper way will lead to better empirical performance.

Continuing with the example in \cref{fig:copula-example}, we assume an autoregressive copula model encodes the probability table on the left. As shown on the right, when provided with the suffix prompt ``in Switzerland'', the copula model alone cannot adjust its probabilities, as it can only condition on prefix prompts. However, a diffusion model that captures the strong dependence between ``Switzerland'' and $Y \!=\! \text{``skiing''}$ can, through I-projection, set the correct marginal probabilities of $Y$, while keeping the copula unchanged. This allows the model to reliably generate “how about alpine skiing.”

Lastly, we need the univariate marginals of $\q (\tilde{\X}_{t} \given \x_{t+1})$, which can be derived by renormalizing $\{\q (X_{t}^{i} \given \x_{t+1})\}_{i}$ to zero out the probability of the mask state according to the following result.


\begin{prop}
\label{prop:aux-marginals-match}
    For each $i$ and data category $c \!\neq\! \text{\texttt{<MASK>}}$, $\q (\tilde{X}_{t}^{i} = c \given \x_{t+1}) \propto \q (X_{t}^{i} = c \given \x_{t+1})$.
    \vspace{-0.4em}
\end{prop}

As a result, for each $i$, the distribution $\p_{\mathrm{dm}} (\tilde{X}_{t}^{i} \given \x_{t+1})$ can be similarly obtained by renormalizing $\p_{\mathrm{dm}} (X_{t}^{i} \given \x_{t+1})$, which is directly obtained from the denoising model, to exclude the mask state.

\subsection{Approximate I-Projection with Autoregressive Models}
\label{sec:Iproj-ar}

\begin{figure}[t]
\vspace{-0.8em}
\begin{algorithm}[H]
\caption{Draw samples from a discrete diffusion model with the help of a copula model}
\label{alg:sampling}
{\fontsize{9}{9} \selectfont
\begin{algorithmic}[1]

\STATE {\bfseries Inputs:} a diffusion model $\p_{\mathrm{dm}}$, a copula model $\p_{\mathrm{copula}}$, number of time steps $T$

\vspace{0.2em}

\STATE {\bfseries Outputs:} a sample $\x_0$ from the discrete diffusion model augmented by the copula model

\vspace{0.2em}

\STATE {\bfseries Initialize:} Sample $\x_{T}$ from the prior noise distribution $\p (\X_{T})$

\vspace{0.2em}

\STATE {\bfseries for} \tikzmarknode{a1}{} $t = T\!-\!1$ {\bfseries to} $0$

\vspace{0.1em}

\STATE \hspace{0.75em} Compute $\{\p_{\mathrm{dm}}(\tilde{X}_{t}^{i} \given \x_{t+1})\}_{i}$ and $\{\p_{\mathrm{dm}}(\tilde{X}_{t}^{i} \given \x_{t+1}^{<i})\}_{i}$ using the diffusion model

\vspace{0.2em}

\STATE \hspace{0.75em} Compute $\V[i,\tilde{x}_{t}^{i}] \!=\! \log \p_{\mathrm{dm}} (\tilde{x}_{t}^{i} \given \x_{t+1}) - \log \p_{\mathrm{dm}} (\tilde{x}_{t}^{i} \given \x_{t+1}^{<i})$ ($\forall i, \tilde{x}_{t}^{i}$) following \cref{eq:copula-update-new}

\vspace{0.1em}

\STATE \hspace{0.75em} Sample $\tilde{\x}_{t}$ from $\hat{\p} (\tilde{\x}_{t} \given \x_{t+1}) \!\propto\! \p_{\mathrm{copula}} (\tilde{\x}_{t} \given \x_{t+1}) \!\cdot\! \prod_{i} \exp ( \V [i, \tilde{x}_{t}^{i}] )$ ($\p_{\mathrm{copula}}$ is defined by \cref{eq:ar-copula-dgm})

\vspace{0.2em}

\STATE \hspace{0.75em} Sample $\x_{t}$ from $\q (\X_{t} \given \tilde{\x}_{t}, \x_{t+1})$ (defined in \cref{prop:mask-can-be-separated})

\end{algorithmic}
}    
\end{algorithm}
\begin{tikzpicture}[overlay,remember picture]
    \draw[black,line width=0.6pt] ([xshift=-12pt,yshift=-4pt]a1.west) -- ([xshift=-12pt,yshift=-52pt]a1.west) -- ([xshift=-8pt,yshift=-52pt]a1.west);
\end{tikzpicture}
\vspace{-3.6em}
\end{figure}

Given univariate marginals $\{\p_{\mathrm{dm}} (\tilde{X}_{t}^{i} \given \x_{t+1})\}_{i}$ and an autoregressive copula distribution $\p_{\mathrm{copula}} (\tilde{\X}_{t} \given \x_{t+1})$, both of which estimate the target distribution $\q(\tilde{\X}_{t} \given \x_{t+1})$, our goal is to combine them following the I-projection procedure described in \cref{sec:copula-general}. Specifically, this involves solving the convex optimization problem in \cref{eq:copula-obj}, which is specialized to the following:
    {\setlength{\abovedisplayskip}{0.4em}
    \setlength{\belowdisplayskip}{0.0em}
    \begin{align}
        \sum_{\tilde{\x}_{t}} \p_{\mathrm{copula}} (\tilde{\x}_{t} \given \x_{t+1}) \cdot \prod_{i} \exp ( \V [i, \tilde{x}_{t}^{i}] ) - \sum_{i = 1}^{N} \sum_{\tilde{x}_{t} = 1}^{C} \V [i, \tilde{x}_{t}^{i}] \cdot \p_{\mathrm{dm}} (\tilde{x}_{t}^{i} \given \x_{t+1}). \label{eq:sp-copula-obj}
    \end{align}}

Following \cref{thm:copula-obj}, if $\V$ minimizes \cref{eq:sp-copula-obj}, then the distribution defined by $\hat{\p} (\tilde{\x}_{t} \given \x_{t+1}) \!=\! \p_{\mathrm{copula}} (\tilde{\x}_{t} \given \x_{t+1}) \!\cdot\! \prod_{i} \exp (\V [i, \tilde{x}_{t}^{i}])$ is the I-projection of $\p_{\mathrm{copula}} (\tilde{\x}_{t} \given \x_{t+1})$ onto the set of distributions with the univariate marginals $\{\p_{\mathrm{dm}} (\tilde{X}_{t}^{i} \given \x_{t+1})\}_{i}$, which is the desired combined distribution.

Consider initializing all coefficients in $\V$ to zero, \ie $\hat{\p} (\tilde{\x}_{t} \given \x_{t+1}) \!=\! \p_{\mathrm{copula}} (\tilde{\x}_{t} \given \x_{t+1})$. For each row $i$, if we only optimize the values $\V [i,:]$ and fix the rest to zero, the optimal coefficients are
    {
    \setlength{\abovedisplayskip}{0.6em}
    \setlength{\belowdisplayskip}{0.2em}
    \begin{align}
        \forall c, \; \V[i, c] = \log \p_{\mathrm{dm}} (\tilde{X}_{t}^{i} = c \given \x_{t+1}) - \log \p_{\mathrm{copula}} (\tilde{X}_{t}^{i} = c \given \x_{t+1}). \label{eq:copula-update}
    \end{align}}

We approximate the solution to \cref{eq:sp-copula-obj} by applying the above update (Eq.~(\ref{eq:copula-update})) to each row in $\V$ independently, as it strikes a proper balance between efficiency and empirical performance.

While the first term on the right-hand side of \cref{eq:copula-update} can be acquired from the diffusion model, the second term is not accessible through the copula model. Plug in the definition in \cref{eq:ar-copula-dgm}, the required marginal probabilities can be written as (for $j \!\in\! J$, $\p_{\mathrm{copula}} (\tilde{x}_{t}^{j} \given \x_{t+1}) \!=\! 1$ iff $\tilde{x}_{t}^{j} \!=\! x_{t+1}^{j}$)
    {\setlength{\abovedisplayskip}{0.6em}
    \setlength{\belowdisplayskip}{0.2em}
    \begin{align}
        \forall i \!\in\! I, \; \p_{\mathrm{copula}} (\tilde{x}_{t}^{i} \given \x_{t+1}) = \p_{\mathrm{copula}} (X_{i} = \tilde{x}_{t}^{i} \given \X_{K_{i}} \!=\! \x_{t+1}^{K_{i}}), \text{~where~} K_{i} \!=\! \{j : j \!\in\! J \text{~and~} j \!<\! i\}. \label{eq:ar-copula-marginal}
    \end{align}}

The above probabilities cannot be computed from the autoregressive model since we need to ``marginalize out'' preceding tokens that are not in $K_{i}$ (\ie those not given as evidence in $\x_{t+1}$). However, these terms can be estimated using the diffusion model. Assume both the diffusion model and the autoregressive model perfectly encode the data distribution. According to \cref{prop:aux-marginals-match}, the diffusion model computes $\p_{\mathrm{dm}} (\tilde{X}_{t}^{i} \given \x_{t+1}) \!=\! \q (\tilde{X}_{t}^{i} \given \x_{t+1})$. Comparing it to \cref{eq:ar-copula-marginal}, which gives $\p_{\mathrm{copula}} (\tilde{X}_{t}^{i} \given \x_{t+1}) \!=\! \q (\tilde{X}_{t}^{i} \given \x_{t+1}^{K_{i}})$, we only need to additionally restrict the diffusion model to only condition on preceding unmasked tokens in $\x_{t+1}$, since $K_{i}$ is the intersection of $J$ and $\{j: j \!<\! i\}$. Therefore, if both models well-approximate the data distribution, we have $\p_{\mathrm{copula}} (\tilde{x}_{t}^{i} \given \x_{t+1}) \!\approx\! \q (\tilde{x}_{t}^{i} \given \x_{t+1}^{K_{i}}) \!=\! \q (\tilde{x}_{t}^{i} \given \x_{t+1}^{<i}) \!\approx\! \p_{\mathrm{dm}} (\tilde{x}_{t}^{i} \given \x_{t+1}^{<i})$, where the equality holds since all values in $\x_{t+1}^{<i}$ but not in $\x_{t+1}^{K_{i}}$ are \texttt{<MASK>}, and does not ``contribute to'' the distribution of $\tilde{X}_{t}^{i}$ according to \cref{prop:mask-can-be-separated}). Correspondingly, we update $\V$ following
    {
    \setlength{\abovedisplayskip}{0.6em}
    \setlength{\belowdisplayskip}{0.2em}
    \begin{align}
        \forall i, c, \; \V[i, c] = \log \p_{\mathrm{dm}} (\tilde{X}_{t}^{i} = c \given \x_{t+1}) - \log \p_{\mathrm{dm}} (\tilde{X}_{t}^{i} = c \given \x_{t+1}^{<i}). \label{eq:copula-update-new}
    \end{align}}

For denoising neural networks that are implemented with bidirectional Transformers, we can simply apply causal attention masks to the self-attention layers to obtain $\{\p_{\mathrm{dm}} (\tilde{X}_{t}^{i} \given \x_{t+1}^{<i})\}_{i}$.

\subsection{The Overall Diffusion Sampling Process}
\label{sec:overall-sampling}

Given a diffusion model $\p_{\mathrm{dm}}$ and an autoregressive copula model $\p_{\mathrm{copula}}$, the sampling procedure is outlined in \cref{alg:sampling}. First, we sample $\x_{T}$ from the prior noise distribution $\p(\X_{T})$ (line 3). During each denoising step $t$, we compute the univariate marginals $\{\p_{\mathrm{dm}} (\tilde{X}_{t}^{i} \given \x_{t+1})\}_{i}$ and $\{\p_{\mathrm{dm}} (\tilde{X}_{t}^{i} \given \x_{t+1}^{<i})\}_{i}$ based on the previously obtained $\x_{t+1}$ (line 5). These marginals are then used to compute the entries in $\V$ (line 6), which approximates the I-projection of $\p_{\mathrm{copula}} (\tilde{\X}_{t} \given \x_{t+1})$ onto the set of distributions with univariate marginals $\{\p_{\mathrm{dm}} (\tilde{X}_{t}^{i} \given \x_{t+1})\}_{i}$ (\cf Sec.~\ref{sec:Iproj-ar}). 

Afterwards, we sample $\tilde{\x}_{t}$ from the combined distribution $\hat{\p} (\tilde{\X}_{t} \given \x_{t+1})$ (line 7). Specifically, following \cref{eq:ar-copula-dgm}, we sample autoregressively following $\hat{\p} (\tilde{\x}_{t} \given \x_{t+1}) \!=\! \prod_{i} \hat{\p} (\tilde{x}_{t}^{i} \given \x_{t+1}, \tilde{\x}_{t}^{<i})$, where
    {
    \setlength{\abovedisplayskip}{0.6em}
    \setlength{\belowdisplayskip}{0.0em}
    \begin{align*}
        \hat{\p} (\tilde{x}_{t}^{i} \given \x_{t+1}, \tilde{\x}_{t}^{<i}) \propto \p_{\mathrm{copula}} (X_{i} = \tilde{x}_{t}^{i} \given \X_{<t} = \tilde{\x}_{t}^{<i}) \cdot \exp(\V [i, \tilde{x}_{t}^{i}]) \cdot \mathbbm{1} [\tilde{x}_{t}^{i} = x_{t+1}^{i}].
    \end{align*}}

Finally, we sample $\x_{t}$ from $\q(\X_{t} \given \tilde{\x}_{t}, \x_{t+1})$ (line 8) as defined in \cref{prop:mask-can-be-separated}. To improve the algorithm's efficiency, we introduce a variant that unmasks tokens in an autoregressive manner. Specifically, at step $t$, all tokens except the first $(T \!-\! t) / T$ portion of the tokens in $\x_{t}$ are converted to \texttt{<MASK>}. Since $\hat{\p}$ is sampled autoregressively, this allows us to use techniques such as KV-caching for autoregressive Transformers \citep{pope2023efficiently} to significantly reduce computation cost introduced by the copula model. See \cref{appx:ar-iproj} for details and the concrete algorithm.

\vspace{-0.5em}
\section{Experiments}
\label{sec:exp}
\vspace{-0.5em}

We empirically validate the proposed method, \textbf{D}iscrete \textbf{C}opula \textbf{D}iffusion (DCD), on language modeling tasks (Sec.~\ref{sec:exp-uncond-text}~and~\ref{sec:exp-cond-text}) and antibody sequence infilling tasks (Sec.~\ref{sec:exp-protein-infill}). For all tasks, we evaluate whether DCD can effectively reduce the number of diffusion steps while maintaining strong performance. Specifically, since DCD combines two pretrained models: a discrete diffusion model and an autoregressive copula model, we examine whether DCD outperforms each individual model.

\vspace{-0.5em}
\subsection{Unconditional Text Generation}
\label{sec:exp-uncond-text}
\vspace{-0.5em}

We first compare the quality of unconditional samples generated by models trained on either WebText \citep{radford2019language} or OpenWebText \citep{gokaslan2019openweb}, which contain web content extracted from URLs shared on Reddit with a minimum number of upvotes. We adopt the medium-sized SEDD model \citep{lou2024discrete} ($\text{SEDD}_{\text{\texttt{M}}}$) since it is a SoTA discrete diffusion model for text generation. The GPT-2-small model \citep{radford2019language} ($\text{GPT-2}_{\text{\texttt{S}}}$) serves as the copula model.

We generate samples of $128$ tokens each. Following \citet{han2023ssd,dieleman2022continuous}, we evaluate sample quality using their generative perplexity, which is the perplexity of the samples when evaluated with the GPT-2-large model. Since previous studies have observed that this metric can be affected by distribution annealing methods such as nucleus sampling, we always sample directly from the models. $\text{SEDD}_{\text{\texttt{M}}}$ is evaluated with $2$ to $256$ diffusion steps and DCD (\ie $\text{SEDD}_{\text{\texttt{M}}}$ with $\text{GPT-2}_{\text{\texttt{S}}}$ as the copula model) is run with diffusion steps ranging from $2$ to $32$. We adopt the log-linear noise schedule suggested by the SEDD paper. See \cref{appx:exp-details-uncond-text} for more details.

\begin{figure}[t]
    \centering
    \begin{minipage}{0.4\columnwidth}
        \includegraphics[width=\linewidth]{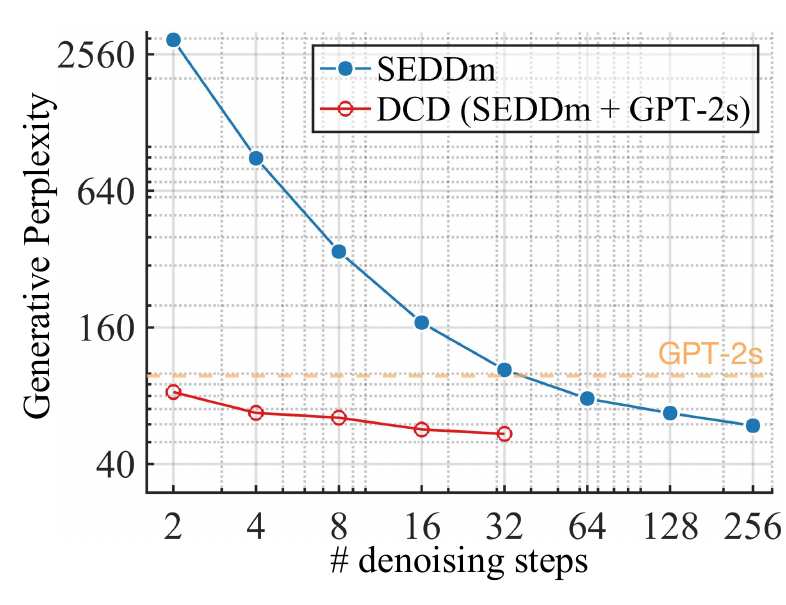}
        \vspace{-2.4em}
        \hspace{0.6em}
        \caption{Generative perplexity ($\downarrow$) with \\ different numbers of denoising steps.}
        \label{fig:uncond-gen}
    \end{minipage}\hfill
    \begin{minipage}{0.55\columnwidth}
        \includegraphics[width=\linewidth]{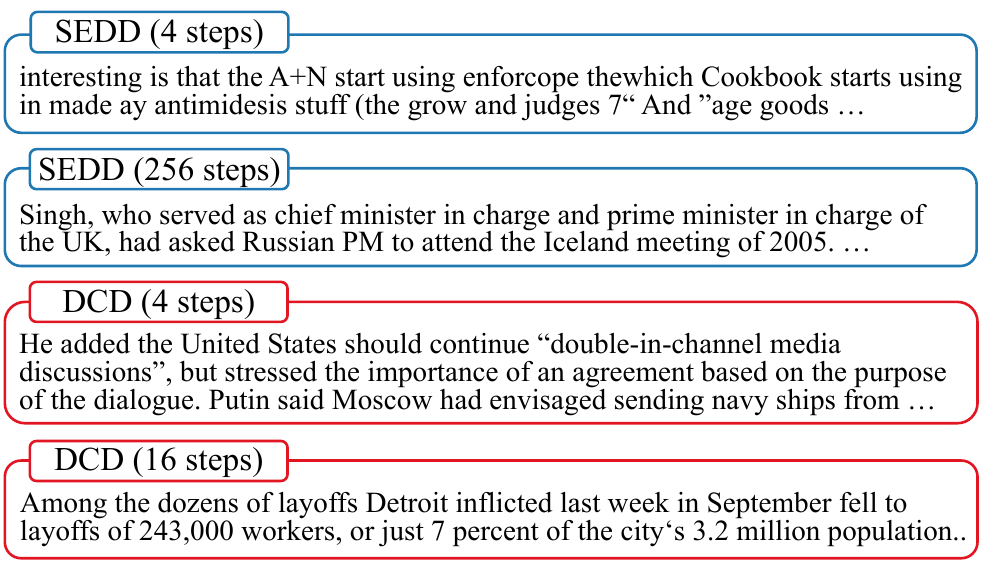}
        \vspace{-2.0em}
        \caption{Generated text from $\text{SEDD}_{\text{\texttt{M}}}$ and DCD with different number of steps. See \cref{appx:additional-text-samples} for more.}
        \label{fig:some-samples}
    \end{minipage}
    \vspace{-0.8em}
\end{figure}

For each configuration, we draw 10,000 samples and report the average perplexity in \cref{fig:uncond-gen}. First, when fixing the number of denoising steps between $2$ to $32$, we observe that DCD outperforms both $\text{SEDD}_{\text{\texttt{M}}}$ with the same number of denoising steps and $\text{GPT-2}_{\text{\texttt{S}}}$. This provides empirical validation of the effectiveness of the I-projection procedure for modeling inter-variable dependencies.

Additionally, DCD with just $4$ denoising steps achieves performance comparable to $\text{SEDD}_{\text{\texttt{M}}}$ with $128$ steps, representing a $32$x reduction in the number of denoising steps. This result not only demonstrates the efficiency of DCD but also underscores the importance of modeling inter-variable dependencies in discrete diffusion models, particularly in few-step generation settings.

Finally, as shown in \cref{fig:some-samples}, SEDD fails to generate fluent and meaningful sentences given only a few diffusion steps, as too many tokens have to be generated in each step. In contrast, by modeling the inter-variable dependencies, DCD generates smooth sentences with only $4$ denoising steps.

\boldparagraph{Efficiency.}
We compare the sample time and the generative perplexity of DCD against competitive baselines in \cref{fig:runtime-performance}. We additionally adopt another recent discrete diffusion baseline MDLM \citep{sahoo2024simple}. We adopt the autoregressive version of DCD as described in \cref{sec:overall-sampling} and \cref{appx:ar-iproj}. Compared to the baselines, DCD consistently achieves better generative perplexity given a fixed runtime constraint. It also requires less time to reach a desired perplexity value. We defer a comprehensive study of DCD's efficiency to \cref{appx:additional-uncond-exps}.


\subsection{Conditional Text Generation}
\label{sec:exp-cond-text}

We now move on to conditional text generation, where certain tokens are provided in advance, and the task is to generate the remaining tokens. As shown in the first column of \cref{tab:wikitext-infilling}, we use five mask strategies, where tokens in specific prompt ranges are given (we use a sequence length of $128$). We adopt the MAUVE score \citep{pillutla2021mauve} with the default settings to compare the difference between the generated and original texts. See \cref{appx:exp-details-cond-text} for further details.

For all methods, we use the same set of 2,000 text sequences from the validation set of WikiText-103 \citep{merity2022pointer}. After applying the prompt mask, we generate $5$ samples for each prompt, resulting in a total number of 10,000 samples.

In addition to $\text{SEDD}_{\text{\texttt{M}}}$ and $\text{GPT-2}_{\text{\texttt{S}}}$, we compare against SSD-LM \citep{han2023ssd}, which is a semi-autoregressive diffusion model designed for text infilling. We adopt the autoregressive unmasking variant of DCD described in the last paragraph of \cref{sec:overall-sampling}.

Results are presented in \cref{tab:wikitext-infilling}. First, DCD outperforms all three baselines in all five tasks. Additionally, when fixing the number of denoising steps between 2 and 32, DCD surpasses both of its base models. Notably, while both $\text{GPT-2}_{\text{\texttt{S}}}$ and the 2-step $\text{SEDD}_{\text{\texttt{M}}}$ performs poorly on the first, the second, and the fifth tasks, combining them in a principled way allows DCD to achieve significantly better performance using only two denoising steps.

\begin{table}[t]
    \centering
    \caption{Evaluation of text infilling performance using the MAUVE score ($\uparrow$) with 5 prompt masks. Scores of DCD are all better than (i) SEDD with the same \# denoising steps, and (ii) $\text{GPT-2}_{\text{\texttt{S}}}$.
    }
    \label{tab:wikitext-infilling}
    \vspace{0.1em}

    \renewcommand{\arraystretch}{1.1}
    \centering
    \scalebox{0.78}{
    \begin{tabular}{cc@{\hspace{0.48em}}ccc@{\hspace{0.48em}}c@{\hspace{0.48em}}c@{\hspace{0.48em}}c@{\hspace{0.48em}}cc@{\hspace{0.48em}}c@{\hspace{0.48em}}c@{\hspace{0.48em}}c@{\hspace{0.48em}}c}
        \toprule
        \multirow{2}{*}[-0.3em]{\makecell{Prompt ranges \\ (remainder is masked)}} & \multicolumn{2}{c}{SSD-LM} & $\text{GPT-2}_{\text{\texttt{S}}}$ & \multicolumn{5}{c}{$\text{SEDD}_{\text{\texttt{M}}}$} & \multicolumn{5}{c}{DCD (ours)} \\
        \cmidrule(lr){2-3}
        \cmidrule(lr){4-4}
        \cmidrule(lr){5-9}
        \cmidrule(lr){10-14}
         & 100 & 500 & N/A & 2 & 4 & 8 & 16 & 32 & 2 & 4 & 8 & 16 & 32 \\
        \cmidrule(lr){1-1}
        \cmidrule(lr){2-14}
        $\scalebox{0.82}{[0.1,0.2] \& [0.5,0.7]}$ & 0.057 & 0.083 & 0.079 & 0.013 & 0.051 & 0.122 & 0.152 & 0.201 & 0.158 & 0.187 & 0.185 & 0.195 & 0.211 \\
        $\scalebox{0.82}{[0.25,0.75]}$ & 0.072 & 0.108 & 0.188 & 0.027 & 0.110 & 0.226 & 0.237 & 0.278 & 0.249 & 0.251 & 0.257 & 0.314 & 0.298 \\
        $\scalebox{0.82}{[0.0,0.1] \& [0.4,0.6] \& [0.9,1.0]}$ & 0.333 & 0.681 & 0.928 & 0.827 & 0.940 & 0.972 & 0.980 & 0.979 & 0.962 & 0.976 & 0.979 & 0.982 & 0.983 \\
        $\scalebox{0.82}{[0.4,0.5] \& [0.8,1.0]}$ & 0.436 & 0.565 & 0.914 & 0.896 & 0.944 & 0.978 & 0.978 & 0.980 & 0.963 & 0.975 & 0.975 & 0.976 & 0.981 \\
        $\scalebox{0.82}{[0.2,0.3] \& [0.6,0.8]}$ & 0.041 & 0.054 & 0.069 & 0.016 & 0.056 & 0.128 & 0.207 & 0.215 & 0.171 & 0.178 & 0.215 & 0.217 & 0.403 \\
        \bottomrule
    \end{tabular}}
    \vspace{-0.4em}
\end{table}


\subsection{Antibody Sequence Infilling}
\label{sec:exp-protein-infill}

We consider the task of unguided antibody infilling, where certain complementarity determining regions (CDRs) of antibodies (\ie sequences of amino acids) are missing and to be generated by the model. We adopt NOC-D \citep{gruver2023protein}, which is a discrete diffusion model trained on 104K antibody sequences from the Observed Antibody Space dataset \citep{ruffolo2023fast}. We further train a GPT model on the same dataset as the copula model. See \cref{appx:exp-details-protein-infill} for training details.

We follow \citet{gruver2023protein} to select the same 10 antibody seed sequences from paired OAS \citep{olsen2022observed}. We consider two infilling tasks: (i) three CDRs $\{\text{HCDR1}, \text{HCDR2}, \text{HCDR3}\}$ are masked, and (ii) two CDRs $\{\text{HCDR1}, \text{LCDR1}\}$ are masked. We follow the original paper and run $64$ diffusion steps for NOS-D. For DCD (\ie combining NOS-D with the trained GPT model as the copula model), we use $4$ denoising steps. We measure the sequence recovery rate, \ie the accuracy of the infilled sequences given the ground truth sequence.

As shown in \cref{tab:protein-infilling}, by combining the univariate marginals from NOS-D and the dependencies captured by the GPT model, DCD can also perform well in antibody sequence infilling tasks.

\begin{figure}[t]
    \centering
    \begin{minipage}{0.48\columnwidth}
        \centering
        \includegraphics[width=0.96\linewidth]{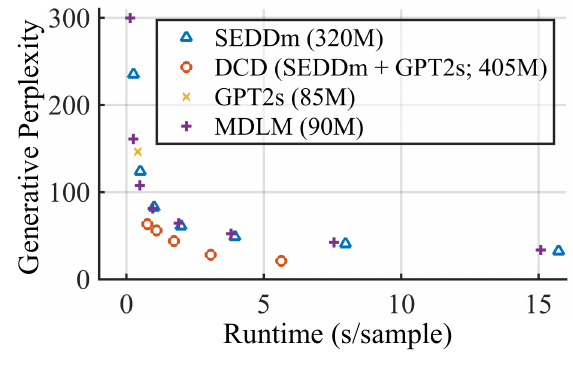}
        \vspace{-1.7em}
        \caption{Sampling time vs. generative perplexity (the autoregressive version of DCD is used).}
        \label{fig:runtime-performance}
    \end{minipage}\hfill
    \begin{minipage}{0.5\columnwidth}
        \caption{Antibody sequence infilling performance measured by sequence recovery rate ($\uparrow$). We compare DCD against its two base models in two tasks, where amino acids at different locations are masked. DCD outperforms both baselines with only 4 denoising steps.}
        \label{tab:protein-infilling}
        \vspace{0.1em}
    
        \renewcommand{\arraystretch}{0.9}
        \centering
        \scalebox{0.84}{
        \begin{tabular}{cccc}
            \toprule
            \multirow{2}{*}[-0.3em]{Method} & \multirow{2}{*}[-0.3em]{\# steps} & \multicolumn{2}{c}{Task} \\
            \cmidrule(lr){3-4}
            & & HCDR\{1+2+3\} & \{H+L\}CDR1 \\
            \midrule
            GPT & N/A & 57.21 & 90.28 \\
            NOS-D & 64 & 51.56 & 88.82 \\
            DCD & 4 & \textbf{58.28} & \textbf{91.58} \\
            \bottomrule
        \end{tabular}}
    \end{minipage}
    \vspace{-0.6em}
\end{figure}

\vspace{-0.4em}
\section{Related Work and Conclusion}
\label{sec:related-work-conclusion}
\vspace{-0.4em}

Diffusion models have been widely applied to model discrete data such as text and DNA sequences. Encouraged by the successes of continuous diffusion models (\eg \citet{ho2020denoising,song2020denoising}), initial attempts convert discrete data into continuous embeddings with either predefined or learned mappings. This enables the use of continuous diffusion models for discrete data \citep{chen2022analog,dieleman2022continuous,li2022diffusion,lovelace2023latent}. However, due to the need for ad-hoc mappings between the discrete data space and the continuous embedding space, which have to be pre-defined or pre-trained, continuous diffusion models are not as effective for modeling discrete distributions \citep{strudel2022self,li2022diffusion,dieleman2022continuous}.

\citet{austin2021structured} proposed the first diffusion model designed directly for discrete data. Later works further improved discrete diffusion models from various aspects such as better loss functions/learning objectives \citep{campbell2022continuous,meng2022concrete,lou2024discrete,benton2022denoising}, better model architectures \citep{sun2022score}, better sampling algorithms \citep{chen2023fast}, and unifying and scaling up existing techniques \citep{sahoo2024simple,shi2024simplified}.

Despite the recent breakthroughs of discrete diffusion models, few papers address the challenge of sampling in a few denoising steps. Some works attribute the failure to perform high-quality few-step generation to a scaling problem of the model. However, we show that the fundamental problem lies in the assumption made by discrete diffusion models that each variable is denoised independently at each step. In addition to identifying this problem, we propose a general solution Discrete Copula Diffusion that combines a discrete diffusion model with a copula model at inference time to obtain a better estimate of the denoising distribution at each step. Concurrently, \citet{guo2024plug} show that energy-based models can also be used as copula models to capture inter-variable dependencies.


There are a few limitations of DCD. First, in addition to a discrete diffusion model, it requires another copula model, which may require additional training for certain applications. Second, although the I-projected distribution is guaranteed as a better estimate of the target distribution, the I-projection step often needs to be approximated in practice. Finally, although DCD requires fewer denoising steps, the computation cost of each step is higher than in discrete diffusion models. Therefore, DCD may not always provide a notable speedup. However, DCD points out the inter-dependency modeling problem and opens up the possibility of combining different types of generative models for better overall performance.

\section*{Acknowledgements}

This work was funded in part by the DARPA ANSR program under award FA8750-23-2-0004, the DARPA CODORD program under award HR00112590089, the Deutsche Forschungsgemeinschaft (DFG, German Research Foundation) under Germany’s Excellence Strategy - EXC 2075 – 390740016, NSF grant \#IIS-1943641, and gifts from Adobe Research and Amazon. We acknowledge the support of the Stuttgart Center for Simulation Science (SimTech). MN thanks IMPRS-IS (International Max Planck Research School for Intelligent Systems) for the support.

\bibliography{iclr2025_conference}
\bibliographystyle{iclr2025_conference}

\newpage

\appendix
\section{Proof of the Theoretical Results}
\label{appx:proofs}

\begin{proof}[Proof of \cref{prop:elbo-decomp}]
    Following \citep{ho2020denoising,sohl2015deep}, the negative ELBO $\calL$ can be decomposed as follows:
        \begin{align*}
            \calL & = \expectation_{\q} \left [ -\log \p(\x_T) - \sum_{t=1}^{T} \log \frac{\p_{\theta} (\x_{t-1} \given \x_{t})}{q(\x_{t} \given \x_{t-1})} \right ], \\
            & = \expectation_{\q} \left [ -\log \p(\x_T) - \sum_{t=1}^{T} \log \frac{\p_{\theta} (\x_{t-1} \given \x_{t}) \cdot \q(\x_{t-1})}{q(\x_{t-1} \given \x_{t}) \cdot \q(\x_{t})} \right ], \\
            & = \expectation_{\q} \left [ -\log \frac{\p(\x_T)}{\q(\x_T)} - \sum_{t=1}^{T} \log \frac{\p_{\theta} (\x_{t-1} \given \x_{t})}{q(\x_{t-1} \given \x_{t})} - \log \p(\x_0) \right ], \\
            & = \kld (\q(\x_{T}) \;\|\; \p(\x_{T})) + \expectation_{q} \left [ \sum_{t=1}^{T} \kld (q(\x_{t-1} \given \x_{t}) \;\|\; \p_{\theta} (\x_{t-1} \given \x_{t})) \right ] + \mathrm{H} (\x_0). \numberthis \label{eq:proof-1-1}
        \end{align*}
    The first term equals $0$ as we assume the noise distribution $\p(\X_{T})$ is consistent in the noising and the denoising processes. Given the independent denoising assumption, when the denoising distribution are optimal, we have
        \begin{align*}
            \forall t \in \{1, \dots, T\}, \; \p_{\theta} (\x_{t-1} \given \x_{t}) = \prod_{i} \q(x_{t-1}^{i} \given \x_{t}).
        \end{align*}
    Plug in \cref{eq:proof-1-1} and using the definition of total correlation, we have:
        \begin{align*}
            \calL & = \kld (\q(\x_{T}) \;\|\; \p(\x_{T})) + \expectation_{q} \left [ \sum_{t=1}^{T} \kld (q(\x_{t-1} \given \x_{t}) \;\|\; \prod_{i} \q(x_{t-1}^{i} \given \x_{t})) \right ] + \mathrm{H} (\x_0) \\
            & = \kld (\q(\x_{T}) \;\|\; \p(\x_{T})) + \sum_{t=1}^{T} \tc (q (\X_{t-1} \given \X_{t})) + \mathrm{H} (\p(\X_0)) \\
            & \geq \mathrm{H} (\p(\X_0)) + \sum_{t=1}^{T} \tc (q (\X_{t-1} \given \X_{t})).
        \end{align*}
\end{proof}

\begin{proof}[Proof of \cref{prop:iproj-is-good}]
    According to Pythagoras' triangle-inequality theorem, if $\hat{\p}$ is the I-projection of $\p_{\mathrm{est}}$ onto $\calP_{\mathrm{mar}}^{\p_{\mathrm{tar}}}$, and $\calP_{\mathrm{mar}}^{\p_{\mathrm{tar}}}$ is convex (this can be shown by applying the definition of a convex set), the following holds for any $\p' \!\in\! \calP_{\mathrm{mar}}^{\p_{\mathrm{tar}}}$:
        \begin{align}
            \kld (\p' \;\|\; \p_{\mathrm{est}}) \geq \kld (\p' \;\|\; \hat{\p}) + \kld (\hat{\p} \;\|\; \p_{\mathrm{est}}). \label{eq:proof-3-1}
        \end{align}
    Choosing $\p' \!=\! \p_{\mathrm{tar}}$, we have
        \begin{align*}
            \kld (\p_{\mathrm{tar}} \;\|\; \hat{\p}) \leq \kld (\p_{\mathrm{tar}} \;\|\; \p_{\mathrm{est}}) - \kld (\hat{\p} \;\|\; \p_{\mathrm{est}}) < \kld (\p_{\mathrm{tar}} \;\|\; \p_{\mathrm{est}}),
        \end{align*}
    \noindent where the last inequality holds since $\kld (\hat{\p} \;\|\; \p_{\mathrm{est}}) \!>\! 0$ if the set of univariate marginals of $\p_{\mathrm{est}}$ and $\p_{\mathrm{tar}}$ are different (as assumed in the proposition).
    
\end{proof}

\begin{proof}[Proof of \cref{prop:phat-form}]
    Following the definition of $\hat{\p}$, we write down the constrained optimization problem as follows
        \begin{align*}
            & \minimize_{\p'}\; \kld (\p' \;\|\; \p_{\mathrm{est}}) \\
            \text{s.t.~} \forall i \in \{1,\dots,N\}&, x_{i} \in \{1, \dots, C\}, \; \sum_{\x_{\backslash i}} \p'(\x_{\backslash i}, x_{i}) = \p_{\mathrm{tar}}(x_{i}).
        \end{align*}
    To incorporate the constraints, we use the method of Lagrange multipliers. The Lagrangian for this problem is
        \begin{align*}
            \calL(\p', \{\lambda_{i}\}_{i=1}^{N}) = \sum_{\x} \p' (\x) \log \frac{\p' (\x)}{\p_{\mathrm{est}} (\x)} + \sum_{i=1}^{N} \sum_{x_i=1}^{C} \lambda_{i} (x_i) \cdot \left ( \sum_{\x_{\backslash i}} \p' (\x_{\backslash i}, x_{i}) - \p_{\mathrm{tar}} (x_i) \right ),
        \end{align*}
    \noindent where the Lagrange multipliers $\{\lambda_{i}\}_{i=1}^{N}$ enforce the univariate marginal constraints.

    To minimize the Lagrangian with respect to $\p' (\x)$, we take the partial derivative of $\calL (\p', \{\lambda_{i}\}_{i=1}^{N})$ with respect to $\p' (\x)$ and set it to $0$:
        \begin{align*}
            \frac{\partial \calL (\p', \{\lambda_{i}\}_{i=1}^{N})}{\partial \p' (\x)} = \log \frac{\p' (\x)}{\p_{\mathrm{est}} (\x)} + 1 + \sum_{i} \lambda_{i} (x_i) = 0.
        \end{align*}
    Simplifying this equation gives
        \begin{align*}
            \p'(\x) = \p_{\mathrm{est}} (\x) \cdot \exp \left ( - 1 - \sum_{i} \lambda_{i} (x_i) \right ).
        \end{align*}
    Defining $\sigma_{i} (x_i) \!:=\! \exp ( - \lambda_{i} (x_i) - 1/N )$ gives $\p'(\x) = \p_{\mathrm{est}} (\x) \prod_{i} \sigma_{i} (x_i)$.

    Existence of the solution follows from the fact that (i) the objective function is convex and bounded (since probability values are in $[0,1]$), and (ii) the set of constraints is feasible (\eg $\p' (\x) \!=\! \prod_{i} \p_{\mathrm{tar}} (x_{i})$ or $\p' (\x) \!=\! \p_{\mathrm{tar}} (\x)$).

\end{proof}

\begin{proof}[Proof of \cref{thm:copula-obj}]
    We show that for any $\V^{*}$ that minimizes the objective function $\calL (\V; \p_{\mathrm{tar}}, \p_{\mathrm{est}})$, the corresponding $\p'$ defined by $\p' (\x) \!=\! \p_{\mathrm{est}} (\x) \cdot \prod_{i} \exp(\V [i, x_i])$ belongs to the set $\calP_{\mathrm{mar}}^{\p_{\mathrm{tar}}}$. Specifically, for any $\V$ that minimizes the objective, the partial derivative of $\calL (\V; \p_{\mathrm{tar}}, \p_{\mathrm{est}})$ with respect to any $\V[i, x_i]$ should be $0$:
        \begin{align*}
            \frac{\partial \calL (\V; \p_{\mathrm{tar}}, \p_{\mathrm{est}})}{\partial \V [i, x_i]} = \exp(\V [i, x_i]) \sum_{\x_{\backslash i}} \p_{\mathrm{est}} (\x_{\backslash i}, x_i) \prod_{j \neq i} \exp(\V [j, x_j]) - \p_{\mathrm{tar}} (x_i) = 0.
        \end{align*}
    Plug in the definition of $\p'$, we have
        \begin{align}
            0 = \sum_{\x_{\backslash i}} \p' (\x_{\backslash i}, x_i) - \p_{\mathrm{tar}} (x_i) = \p' (x_i) - \p_{\mathrm{tar}} (x_i). \label{eq:proof-4-1}
        \end{align}
    Since \cref{eq:proof-4-1} holds for all $(i, x_i)$ pairs, we have that every minimizer of $\calL (\V; \p_{\mathrm{tar}}, \p_{\mathrm{est}})$ corresponds to a distribution $\p'$ in $\calP_{\mathrm{mar}}^{\p_{\mathrm{tar}}}$. Since $\calL (\V; \p_{\mathrm{tar}}, \p_{\mathrm{est}})$ is convex, we can also argue the converse: if a distribution $\p'$ with the above-defined form belongs to $\calP_{\mathrm{mar}}^{\p_{\mathrm{tar}}}$, then the corresponding $\V$ is a minimizer of $\calL (\V; \p_{\mathrm{tar}}, \p_{\mathrm{est}})$.

    According to \cref{prop:phat-form}, the solution to the following I-projection exists and its solution $\hat{\p}$ has the same form as $\p'$.
        \begin{align*}
            \hat{\p} = \argmin_{\p' \in \calP_{\mathrm{mar}}^{\p}} \kld (\p' \;\|\; \p_{\mathrm{est}}).
        \end{align*}
    Since $\hat{\p}$ has the same form as $\p'$ (by Prop.~\ref{prop:phat-form}) and belongs to $\calP_{\mathrm{mar}}^{\p_{\mathrm{tar}}}$, it is the a minimizer of $\calL (\V; \p_{\mathrm{tar}}, \p_{\mathrm{est}})$.
\end{proof}

\begin{proof}[Proof of \cref{prop:copula_is_invariant}]
The copula of $p$ is shown to be invariant under rescalings of the form $q(\bm{x})\propto\p (\x) \cdot \prod_{i} \exp (\V [i, x_i])$ for any $\V\in \mathbb{R}^{N\times C}$ by using the parameterization of a discrete copula by conditional odds ratios (\cref{def:conditional_odds_ratios}). The scaling factors cancel in the ratios as shown, e.g. by \citet[Theorem 12.3]{rudas2018lectures}.
\end{proof}

\begin{proof}[Proof of \cref{prop:mask-can-be-separated}]
    We start by writing the probability $\q (\x_{t} \given \x_{t+1})$ using the Bayes' rule:
        \begin{align*}
            \q (\x_{t} \given \x_{t+1}) & = \q (\x_{t+1} \given \x_{t}) \cdot \frac{\q(\x_{t})}{\q(\x_{t+1})}, \\
            & = \sum_{\x_{0}} \frac{1}{\q (\x_{t+1})} \cdot \q (\x_{t+1} \given \x_{t}) \cdot \q(\x_{t} \given \x_{0}) \cdot \p (\x_{0}), \numberthis \label{eq:proof-5-1}
        \end{align*}
    \noindent where the last equality follows from $\q (\x_{t}) \!=\! \sum_{\x_{0}} \q (\x_{t} \given \x_{0}) \!\cdot\! \p (\x_{0})$. Recall from the proposition that $I$ is defined as the set of variables $i$ such that $x_{t+1}^{i} \!=\! \text{\texttt{<MASK>}}$ and $J$ is the complement of $I$.

    First, we must have $x_{t}^{j} \!=\! x_{t+1}^{j}$ for $j \!\in\! J$ since for any other value of $X_{t}^{j}$, we have $\q(\x_{t+1} \given \x_{t}) \!=\! 0$ in \cref{eq:proof-5-1}. As a result, $\q(\x_{t} \given \x_{t+1})$ is also zero.

    We then move our attention to the variables in $I$. We first consider the probability $\q(X_{t}^{i} \!=\! \text{\texttt{<MASK>}} \given \x_{t+1})$ for any $i \!\in\! I$. Following \cref{eq:proof-5-1}, we have
        \begin{align*}
            \q (X_{t}^{i} = \text{\texttt{<MASK>}} \given \x_{t+1}) & = \sum_{\x_0} \sum_{\x_{t}^{\backslash i}} \frac{1}{\q (\x_{t+1})} \cdot \q (\x_{t+1} \given \x_{t}) \cdot \q(\x_{t} \given \x_{0}) \cdot \p (\x_{0}), \\
            & = \sum_{\x_{t}^{\backslash i}} \frac{1}{\q (\x_{t+1})} \cdot \q (\x_{t+1} \given \x_{t}) \cdot \q(\x_{t}), \\
            & = \frac{\q (X_{t+1}^{i} = \text{\texttt{<MASK>}} \given X_{t}^{i} = \text{\texttt{<MASK>}}) \cdot \q (X_{t}^{i} = \text{\texttt{<MASK>}})}{\q (X_{t+1}^{i} = \text{\texttt{<MASK>}})}, \\
            & = \frac{\q (X_{t}^{i} = \text{\texttt{<MASK>}})}{\q (X_{t+1}^{i} = \text{\texttt{<MASK>}})} = \frac{\alpha_{t}}{\alpha_{t+1}}. \numberthis \label{eq:proof-5-2}
        \end{align*}
    We then focus on $\X_{t}^{I} \!=\! \x_{t}^{I}$, where none of the value in $\x_{t}^{I}$ is \texttt{<MASK>}. Note that we also need to have $\X_{t}^{J} \!=\! \x_{t+1}^{J}$. 
        \begin{align*}
            \q (\x_{t} \given \x_{t+1}) & \propto \sum_{\x_{0}} \q (\x_{t+1} \given \x_{t}) \cdot \q(\x_{t} \given \x_{0}) \cdot \q (\x_{0}), \\
            & \overset{(a)}{=} \q (\x_{t+1} \given \x_{t}) \cdot \q(\X_{0} = \x_{t}), \\
            & = \left ( \frac{\alpha_{t+1} - \alpha_{t}}{1 - \alpha_{t}} \right )^{\abs{I}} \cdot \q(\X_{0} = \x_{t}), \\
            & \propto \q(\X_{0} = \x_{t}),  \numberthis \label{eq:proof-5-3}
        \end{align*}
    \noindent where $\p (\X_{0})$ is the data distribution; $(a)$ follows from the fact that no value in $\x_{t}$ is \texttt{<MASK>}, hence $\x_{0} \!=\! \x_{t}$; $\frac{\alpha_{t+1} - \alpha_{t}}{1 - \alpha_{t}}$ is the probability of transitioning into the mask state from time $t$ to time $t \!+\! 1$.

    Denote $\tilde{\X}_{t}$ as a set of variables with the same configuration and semantics as $\X_{t}$, with the only difference that the category \texttt{<MASK>} is excluded. By following \cref{eq:proof-5-3} and apply normalization, we conclude that
        \begin{align}
            \q (\tilde{\x}_{t} \given \x_{t+1}) = \p(\X_{0}^{I} = \tilde{\x}_{t}^{I} \given \X_{0}^{J} = \x_{t+1}^{J}) \cdot \mathbbm{1} [\tilde{\x}_{t}^{J} = \x_{t+1}^{J}]. \label{eq:proof-5-3b}
        \end{align}
    This matches the definition in \cref{eq:aux-denoising-dist}. 

    Finally, we verify the correctness of the distribution $\q (\X_{t} \given \tilde{\x}_{t}, \x_{t+1})$ defined in the proposition by verifying the following for any $\x_{t}$
        \begin{align}
            \q (\x_{t} \given \x_{t+1}) = \sum_{\tilde{\x}_{t}} \q (\tilde{\x}_{t} \given \x_{t+1}) \cdot \q (\x_{t} \given \tilde{\x}_{t}, \x_{t+1}). \label{eq:proof-5-4}
        \end{align}
    Denote $K$ as the set of variables $i$ such that $\x_{t} \!=\! \text{\texttt{<MASK>}}$ and $L$ as its complement. First, if $L \subseteq J$ (\ie $I \subseteq K$), then both the left-hand side (LHS) and the right-hand sides (RHS) are zero. Specifically, the RHS is zero since according to the definition, $\forall i \!\in\! J \,\&\, i \!\in\! K$, we have $\q (x_{t}^{i} \given \tilde{x}_{t}^{i}, x_{t+1}^{i}) \!=\! 0$.

    Next, if $K \subseteq I$, we can decompose $\q (\x_{t} \given \x_{t+1})$ as follows
        \begin{align}
            \q (\x_{t} \given \x_{t+1}) = \q (\x_{t}^{I \backslash K} \given \x_{t+1}) \cdot \prod_{i \in K} \q (x_{t}^{i} \given \x_{t+1}) \cdot \prod_{j \in J} \q (x_{t}^{j} \given \x_{t+1}). \label{eq:proof-5-5}
        \end{align}
    For any $j \!\in\! J$, if $x_{t}^{j} \!\neq\! x_{t+1}^{j}$ then both the LHS and the RHS of \cref{eq:proof-5-4} are zero. Otherwise we always have $\q (x_{t}^{j} \given \x_{t+1}) \!=\! 1$. Therefore, \cref{eq:proof-5-5} can be further simplified as
        \begin{align}
            \q (\x_{t} \given \x_{t+1}) = \q (\x_{t}^{I \backslash K} \given \x_{t+1}) \cdot \prod_{i \in K} \q (x_{t}^{i} \given \x_{t+1}). \label{eq:proof-5-6}
        \end{align}
    We then proceed to simplify the RHS of \cref{eq:proof-5-4}:
        \begin{align*}
            & \sum_{\tilde{\x}_{t}} \q (\tilde{\x}_{t} \given \x_{t+1}) \cdot \q (\x_{t} \given \tilde{\x}_{t}, \x_{t+1}), \\
            = & \sum_{\tilde{\x}_{t}^{K}} \q (\tilde{\x}_{t}^{K}, \tilde{\x}_{t}^{I \backslash K} \given \x_{t+1}) \cdot \left ( \frac{\alpha_{t}}{\alpha_{t+1}} \right )^{\abs{K}} \cdot \left ( \frac{\alpha_{t+1} - \alpha_{t}}{\alpha_{t+1}} \right )^{\abs{I} - \abs{K}}, \\
            \overset{(a)}{=} & \sum_{\tilde{\x}_{t}^{K}} \q (\tilde{\x}_{t}^{K}, \tilde{\x}_{t}^{I \backslash K} \given \x_{t+1}) \cdot \left ( \frac{\alpha_{t+1} - \alpha_{t}}{\alpha_{t+1}} \right )^{\abs{I} - \abs{K}} \cdot \prod_{i \in K} \q (x_{t}^{i} \given \x_{t+1}), \\
            = & \, \q (\tilde{\x}_{t}^{I \backslash K} \given \x_{t+1}) \cdot \left ( \frac{\alpha_{t+1} - \alpha_{t}}{\alpha_{t+1}} \right )^{\abs{I} - \abs{K}} \cdot \prod_{i \in K} \q (x_{t}^{i} \given \x_{t+1}), \\
            \overset{(b)}{\propto} & \, \p (\X_{0}^{I \backslash K} = \tilde{\x}_{t}^{I \backslash K}, \X_{0}^{J} = \tilde{\x}_{t}^{J}) \cdot \prod_{i \in K} \q (x_{t}^{i} \given \x_{t+1}), \\
            \overset{(c)}{\propto} & \, \q (\X_{t}^{I \backslash K} = \tilde{\x}_{t}^{I \backslash K} \given \x_{t+1}) \cdot \prod_{i \in K} \q (x_{t}^{i} \given \x_{t+1}), \numberthis \label{eq:proof-5-7}
        \end{align*}
    \noindent where $(a)$ follows from \cref{eq:proof-5-2}, $(b)$ applies the definition in \cref{eq:proof-5-3b}, and $(c)$ is a result of applying \cref{eq:proof-5-3} to the case where $\tilde{\x}_{t}^{L} \!=\! \{\tilde{\x}_{t}^{I \backslash K}, \tilde{\x}_{t}^{J}\}$ are not \texttt{<MASK>}.

    By combining \cref{eq:proof-5-6,eq:proof-5-7}, we conclude that the LHS and the RHS of \cref{eq:proof-5-4} are proportional to each other. Since they are both properly-normalized distributions, they must also match exactly.
    
\end{proof}

\begin{proof}[Proof of \cref{prop:aux-marginals-match}]
    We first state a more detailed version of the proposition: for each variable $i$ and data category $c$ ($c \!\neq\! \text{\texttt{<MASK>}}$), we have
        \begin{align*}
            \q (\tilde{X}_{t}^{i} = c \given \x_{t+1}) = \frac{1}{Z} \cdot \q (X_{t}^{i} = c \given \x_{t+1}), \text{~where~} Z = \sum_{c \neq \text{\texttt{<MASK>}}} \q (X_{t}^{i} = c \given \x_{t+1}).
        \end{align*}
    According to the proof of \cref{prop:mask-can-be-separated}, \cref{eq:proof-5-4} holds for all $\x_{t}$. Therefore, we have that for each $i$ and each data category $x_{t}^{i} \!\neq\! \text{\texttt{<MASK>}}$,
        \begin{align}
            \q (x_{t}^{i} \given \x_{t+1}) = \sum_{\tilde{\x}_{t}} \q (\tilde{\x}_{t} \given \x_{t+1}) \cdot \q (x_{t}^{i} \given \tilde{\x}_{t}, \x_{t+1}). \label{eq:proof-6-1}
        \end{align}
    If $i \!\in\! J$, then both the LHS of the above equation and $\q (x_{t}^{i} \given \tilde{\x}_{t}, \x_{t+1})$ equals one if and only if $x_{t}^{i} \!=\! x_{t+1}^{i}$. Therefore, the result holds trivially.

    Next, if $i \!\in\! I$, denote $I_{\backslash i} \!:=\! I \backslash \{i\}$, \cref{eq:proof-6-1} is simplified to 
        \begin{align*}
            \q (x_{t}^{i} \given \x_{t+1}) & = \sum_{\tilde{\x}_{t}} \q (\tilde{\x}_{t} \given \x_{t+1}) \cdot \q (x_{t}^{i} \given \tilde{\x}_{t}, \x_{t+1}), \\
            & = \sum_{\tilde{x}_{t}^{i}} \sum_{\tilde{\x}_{t}^{I_{\backslash i}}} \q (\tilde{x}_{t}^{i}, \tilde{\x}_{t}^{I_{\backslash i}} \given \x_{t+1}) \cdot \q (x_{t}^{i} \given \tilde{x}_{t}^{i}, x_{t+1}^{i}), \\
            & = \q(\tilde{X}_{t}^{i} = x_{t}^{i} \given \x_{t+1}) \cdot \q (x_{t}^{i} \given \tilde{X}_{t}^{i} = x_{t}^{i}, x_{t+1}^{i}), \\
            & = \q(\tilde{X}_{t}^{i} = x_{t}^{i} \given \x_{t+1}) \cdot \frac{\alpha_{t+1} - \alpha_{t}}{\alpha_{t+1}}.
        \end{align*}
    Therefore, we have
        \begin{align*}
            \q (\tilde{X}_{t}^{i} = x_{t}^{i} \given \x_{t+1}) = \frac{1}{Z} \cdot \q (X_{t}^{i} = x_{t}^{i} \given \x_{t+1}), \text{~where~} Z = \sum_{x_{t}^{i} \neq \text{\texttt{<MASK>}}} \q (X_{t}^{i} = x_{t}^{i} \given \x_{t+1}).
        \end{align*}
\end{proof}

\section{Relation Between $\calL (\V; \p_{\mathrm{tar}}, \p_{\mathrm{est}})$ and Matrix Scaling}
\label{appx:relation}

The matrix scaling problem gives a matrix $A$ as input and asks for diagonal `scaling' matrices $X$ and $Y$ such that $XAY$ is doubly stochastic (its row and column sums are all one). More generally, target row and column sum vectors $r$ and $c$ are provided and need not contain only ones. The solvability of this problem for positive matrices was established by \cite{sinkhorn1964relationship}, and its algorithms (sometimes called iterative proportional fitting), generalizations, and numerous applications have been studied thoroughly \citep{kalantari1993rate,ruschendorf1995convergence,allen2017much}; see \citep{idel2016review} for a review. Taking the multidimensional generalization of the problem and interpreting the tensor as a (unnormalized) probability distribution yields the connection to our problem, with the target sums being the univariate marginal distributions.

\section{Parameterizing Discrete Copulas by Odds Ratios}
\label{appx:discrete-copula}


We start by formally defining odds ratios.

\begin{defn}[\cite{rudas2018lectures}]\label{def:conditional_odds_ratios}
Let $p$ be a distribution over variables $\bm{X}$ each taking values in $\{0,1\}$. For a partition of $\bm{X}$ into sets $\bm{A}$ and $\bm{B}$, the \emph{conditional odds ratio} of variables $\bm{A}$ conditioned on the assignment $\bm{B}=\bm{b}$ is 
\[
\text{COR}_p(\bm{A}|\bm{B}=\bm{b})=\frac{\prod_{\bm{a}\in s} p(\bm a, \bm b)}{\prod_{\bm{a}\in d}p(\bm a, \bm b)}
\]
where $s$ is the set of assignments to $\bm{A}$ whose parity is the same as the number of variables in $\bm{A}$, and $d$ is the set of assignments whose parity is different.
\end{defn}

In the case of more than two categories per variable, $\text{COR}_p(\bm{A}|\bm{B}=\bm{b})$ can generalized further to be a set of similarly defined ratios (see, e.g., \citet{rudas2018lectures}). Together the set of all conditional odds ratios $\text{COR}_p(\bm{A}|\bm{B}=\bm{b})$ for partitions of $\bm{X}$ into sets $\bm{A}$ and $\bm{B}$ with $|\bm{A}|\ge 2$, completely specifies the association among the variables in the joint distribution $p$, as established by the following theorem.

\begin{thm}[\cite{rudas2018lectures}]\label{thm:discrete_combine_copula}
Let $q$ and $r$ be positive probability distributions on a the set of variables $\bm{X}$ each taking values in $\{0,1,\ldots,k\}$. Then there exists a unique probability distribution $p$ such that $p$ has the same univariate marginal distributions as $q$, that is, for all $i$
\[
p(x_i)=q(x_i),
\]
and $p$ has the same copula as $q$, that is for all partitions of $\bm{X}$ into sets $\bm{A}$ and $\bm{B}$ with $|\bm{A}|\ge 2$,
\[
\text{COR}_p(\bm{A}|\bm{B}=\bm{b})=\text{COR}_r(\bm{A}|\bm{B}=\bm{b}).
\]
\end{thm}

\begin{proof}
This follows from \citep[Theorem 10.2]{rudas2018lectures} by taking the descending set to contain the empty set and all singletons (and the ascending set, its complement).
\end{proof}

Theorem~\ref{thm:discrete_combine_copula} shows how any distribution $p$ can be viewed as combining independent marginal distributions (i.e., from $r$) and odds ratios (i.e., from $q$).
Such a combination has desirable properties. For example, in the case of two variables with possibly many categories, it has been shown that among all distributions with the same margins as $r$, the distribution $p$ minimizes the KL-divergence to $q$ \citep[Theorem 6.2]{geenens2020copula}, i.e. that $p$ is the information projection of $q$ onto the set of distributions with the margins of $r$.

\section{Unbiased Univariate Marginals from Discrete Diffusion Models}
\label{appx:dm-unbiased}

In this section, we show that when their respective training losses are minimized, discrete-time and continuous-time discrete diffusion models recover the true univariate marginals.

\boldparagraph{Discrete-Time Diffusion Models.}
Discrete-time diffusion models \citep{austin2021structured} are trained to maximize the ELBO between the forward joint distribution $\p(\x_{0}) \q(\x_{1:T} \given \x_{0})$, where $\p(\x_{0})$ is the data distribution, and the reverse joint distribution $\p_{\theta} (\x_{0:T})$. The ELBO can be simplified to 
    \begin{align*}
        \expectation_{\q} \left [ \log \frac{\p(\x_{T})}{\q(\x_{T})} + \sum_{t=1}^{T} \log \frac{\p_{\theta}(\x_{t-1} \given \x_{t})}{\q(\x_{t-1} \given \x_{t})} + \log \p(\x_{0}) \right ].
    \end{align*}
Assume that $\p_{\theta} (\x_{t-1} \given \x_{t})$ encodes fully-factorized distribution, the above objective can be simplified as
    \begin{align*}
        \sum_{t=1}^{T} \sum_{i} \q (x_{t-1}^{i} \given \x_{t}) \log \frac{\p_{\theta} (x_{t-1}^{i} \given \x_{t})}{\q (x_{t-1}^{i} \given \x_{t})} + 
        \expectation_{\q} \left [ \log \frac{\p(\x_{T})}{\q(\x_{T})} + \log \p(\x_{0}) \right ],
    \end{align*}
\noindent where the second term is independent to $\p_{\theta}$. From the first term of the above formula, we can conclude that the ELBO objective is maximized when $\p_{\theta} (x_{t-1}^{i} \given \x_{t}) \!=\! \q (x_{t-1}^{i} \given \x_{t})$ for every $t$ and every $i$.

\boldparagraph{Continuous-Time Diffusion Models.} As described in \cref{sec:preliminaries}, many continuous-time diffusion models learn to approximate the likelihood ratio (defined as $s_{\theta} (\x_{t}, \x'_{t}; t)$) at all noise levels $t \!\in\! [0, T]$:
    \begin{align*}
        s_{\theta} (\x_{t}, \x'_{t}; t) := \frac{\q(\X_{t} = \x'_{t})}{\q(\X_{t} = \x_{t})}.
    \end{align*}
Specifically, \citet{lou2024discrete,meng2022concrete} directly parameterize a neural network to approximate the likelihood ratios, and \citet{sun2022score} approximates the likelihood ratios with the conditional distributions $\p_{\theta} (X_{t}^{i} \given \x_{t}^{\backslash i})$ ($\forall i, t$).

For each $\x_{t}$, since there are exponentially many possible $\x'_{t}$, it is infeasible to have a neural network to directly model the likelihood ratio for all pairs of $(\x_{t}, \x'_{t})$. Instead, they focus on $(\x_{t}, \x'_{t})$ pairs where $\x_{t}$ and $\x'_{t}$ are only different in one single variable, \ie their Hamming distance is one. For example, in \citet{lou2024discrete}, they represent $s_{\theta}$ as $s_{\theta} (\x_{t}, y^{i}_{t}; t, i)$, which computes the likelihood ratio between $\x_{t}$ and $\x'_{t} \!=\! \{\x_{t}^{\backslash i}, y^{i}_{t}\}$. $s_{\theta}$ is trained by minimizing the following objective:
    \begin{align*}
        \expectation_{t, \x_{t} \sim \q(\X_{t})} \left [ \sum_{i} \sum_{y_{t}^{i} \neq x_{t}^{i}} w_{t} \left ( s_{\theta} (\x_{t}, y_{t}^{i}; t, i) - \frac{\q(\X_{t} = \{\x_{t}^{\backslash i}, y_{t}^{i}\})}{\q (\X_{t} = \x_{t})} \log s_{\theta} (\x_{t}, y_{t}^{i}; t, i) \right ) \right ],
    \end{align*}
\noindent where $\{w_{t}\}_{t}$ are positive weights. When the above objective is minimized, $s_{\theta}$ recovers the correct likelihood ratios:
    \begin{align}
        \forall i, t, \; s_{\theta} (\x_{t}, y_{t}^{i}; t, i) = \frac{\q(\X_{t} = \{\x_{t}^{\backslash i}, y_{t}^{i}\})}{\q (\X_{t} = \x_{t})}. \label{eq:appx-c-1}
    \end{align}
At inference time, continuous-time discrete diffusion models select a list of time steps $0 \!<\! t_{0} \!<\! \cdots \!<\! t_{k} \!=\! T$ to sample from: first sample from the prior $\p (\X_{t_{k}})$ and then sample recursively from $\{\p_{\theta} (\x_{t_{i-1}} \given \x_{t_{i}})\}_{i=1}^{k}$, where $\p_{\theta} (\x_{t_{i-1}} \given \x_{t_{i}})$ is obtained from $s_{\theta} (\x_{t}, y_{t}^{i}; t, i)$ in an indirect manner. Specifically, assume $\frac{d \p(\x_{t})}{d t} \!=\! Q \!\cdot\! \p(\x_{t})$, we have\footnote{This argument largely follows Theorem 4.1 in \citet{lou2024discrete}. We include it for the sake of completeness.}
    \begin{align*}
        \q (\x_{t_{i-1}} \given \x_{t_{i}}) & = \q(\x_{t_{i}} \given \x_{t_{i-1}}) \cdot \frac{\q (\x_{t_{i-1}})}{\q (\x_{t_{i}})}, \\
        & = \q(\x_{t_{i}} \given \x_{t_{i-1}}) \cdot \left ( \sum_{\x} \exp (- \Delta t \!\cdot\! Q) (\x_{t_{i-1}}, \x) \cdot \frac{\q (\X_{t_{i}} = \x)}{\q (\X_{t_{i}} = \x_{t_{i}})} \right ),
    \end{align*}
\noindent where $\Delta t \!:=\! t_{i} \!-\! t_{i-1}$ and $\exp (- \Delta t \!\cdot\! Q) (\x_{t_{i-1}}, \x)$ denotes the product of $\exp (- \Delta t \!\cdot\! Q) (x_{t_{i-1}}^{j}, x^{j})$, the $x_{t_{i-1}}^{j}$-th row and $x^{j}$-th column of $\exp (- \Delta t \!\cdot\! Q)$.

Plug in \cref{eq:appx-c-1}, we can compute the marginal of $x_{t_{i-1}}^{j}$ (\ie $\p_{\theta} (x_{t_{i-1}}^{j} \given \x_{t_{i}})$) following
    \begin{align*}
        \q (X_{t_{i-1}}^{j} = y \given \x_{t_{i}}) & \propto \q(\x_{t_{i}} \given \x_{t_{i-1}}) \cdot \left ( \sum_{y'} \exp (- \Delta t \!\cdot\! Q) (y, y') \cdot s_{\theta} (\x_{t_{i}}, y'; t_{i}, j) \right ), \\
        & = \exp (\Delta t \!\cdot\! Q) (y, x_{t_{i}}^{j}) \cdot \left ( \sum_{y'} \exp (- \Delta t \!\cdot\! Q) (y, y') \cdot s_{\theta} (\x_{t_{i}}, y'; t_{i}, j) \right ).
    \end{align*}
Therefore, if $s_{\theta}$ perfectly learns the likelihood ratios between inputs with Hamming distance at most one, then the correct marginals $\q (x_{t_{i-1}}^{j} \given \x_{t_{i}})$ can be computed using $s_{\theta}$.

\section{Implementation Details of DCD}
\label{appx:ar-iproj}

\begin{figure}[t]
    \centering
    \includegraphics[width=0.6\columnwidth]{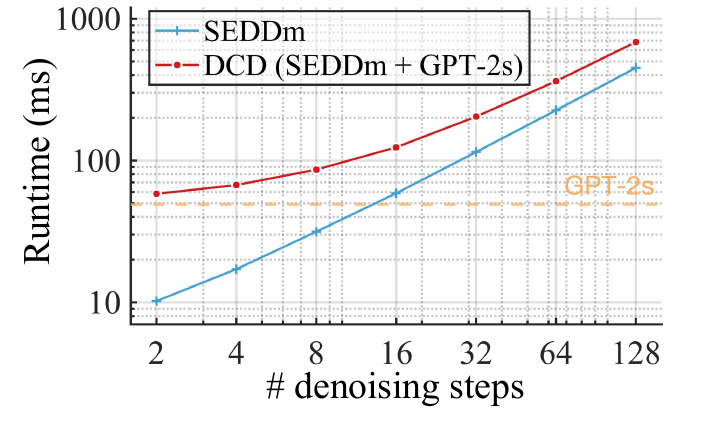}
    \vspace{-1.7em}
    \caption{Sampling time of DCD and its two base models with 2 to 128 denoising steps.}
    \label{fig:runtime}
\end{figure}

We describe details about the ``autoregressive'' version of DCD introduced in \cref{sec:overall-sampling}. According to \cref{sec:overall-sampling}, the first $(T \!-\! t \!-\! 1) / T$ portion of the tokens in $\x_{t+1}$ are unmasked. At step $t$, we only need to additionally unmask the tokens spanning the $(T \!-\! t \!-\! 1) / T$ to $(T \!-\! t) / T$ fraction of the sequence $\x_{t}$. We do this by caching the keys and values generated by the attention layers of tokens generated in previous denoising steps. So at step $t$, we will have the KV-caches of the first $(T \!-\! t \!-\! 1) / T$ fraction of tokens. As a result, the computational cost for running the autoregressive Transformer is independent of the number of denoising steps.

\begin{figure}[H]
\vspace{-0.8em}
\begin{algorithm}[H]
\caption{DCD with Autoregressive Copula Models and Using Autoregressive Sampling}
\label{alg:sampling-gpt-ar}
{\fontsize{9}{9} \selectfont
\begin{algorithmic}[1]

\STATE {\bfseries Inputs:} a diffusion model $\p_{\mathrm{dm}}$, an autoregressive model $\p_{\mathrm{ar}}$, number of time steps $T$, sequence length $L$

\vspace{0.2em}

\STATE {\bfseries Outputs:} a sample $\x_0$ from the discrete diffusion model augmented by the autoregressive model

\vspace{0.2em}

\STATE {\bfseries Initialize:} Sample $\x_{T}$ from the prior noise distribution $\p (\X_{T})$

\vspace{0.2em}

\STATE {\bfseries for} \tikzmarknode{a1}{} $t = T\!-\!1$ {\bfseries to} $0$

\vspace{0.1em}

\STATE \hspace{0.75em} $i_{\mathrm{min}}, i_{\mathrm{max}} \!=\! \frac{L}{T} \!\cdot\! (T \!-\! t \!-\! 1), \frac{L}{T} \!\cdot\! (T \!-\! t)$ (w.l.o.g. assume $L$ is divisible by $T$)

\STATE \hspace{0.75em} Compute $\{\p_{\mathrm{dm}}(\tilde{X}_{t}^{i} \given \x_{t+1})\}_{i}$ and $\{\p_{\mathrm{dm}}(\tilde{X}_{t}^{i} \given \x_{t+1}^{<i})\}_{i}$ for each $i \in [i_{\mathrm{min}}, i_{\mathrm{max}})$ using the diffusion model

\vspace{0.2em}

\STATE \hspace{0.75em} Compute $\V[i,\tilde{x}_{t}^{i}] \!=\! \log \p_{\mathrm{dm}} (\tilde{x}_{t}^{i} \given \x_{t+1}) - \log \p_{\mathrm{dm}} (\tilde{x}_{t}^{i} \given \x_{t+1}^{<i})$ ($\forall i \in [i_{\mathrm{min}}, i_{\mathrm{max}}), \tilde{x}_{t}^{i}$)

\vspace{0.2em}

\STATE \hspace{0.75em} $\x_{t} \leftarrow \x_{t+1}$

\vspace{0.1em}

\STATE \hspace{0.75em} {\bfseries for} \tikzmarknode{a2}{} $i = i_{\mathrm{min}}$ {\bfseries to} $i_{\mathrm{max}} - 1$

\vspace{0.2em}

\STATE \hspace{1.5em} Sample $x_{t}^{i}$ from $\hat{\p} (x_{t}^{i}) \propto \p_{\mathrm{ar}} (x_{t}^{i} \given \x_{t}^{<i}) \cdot \prod_{i} \exp ( \V [i, x_{t}^{i}] )$ and store it to $\x_{t}$

\end{algorithmic}
}    
\end{algorithm}
\begin{tikzpicture}[overlay,remember picture]
    \draw[black,line width=0.6pt] ([xshift=-12pt,yshift=-4pt]a1.west) -- ([xshift=-12pt,yshift=-72pt]a1.west) -- ([xshift=-8pt,yshift=-72pt]a1.west);
\end{tikzpicture}
\begin{tikzpicture}[overlay,remember picture]
    \draw[black,line width=0.6pt] ([xshift=-12pt,yshift=-4pt]a2.west) -- ([xshift=-12pt,yshift=-12pt]a2.west) -- ([xshift=-8pt,yshift=-12pt]a2.west);
\end{tikzpicture}
\vspace{-3.2em}
\end{figure}

\boldparagraph{Additional Runtime Analysis.}
\cref{fig:runtime} displays the generation time per sample for $\text{SEDD}_{\text{\texttt{M}}}$, $\text{GPT-2}_{\text{\texttt{S}}}$, and DCD. When the number of denoising steps is small, the computation cost of running $\text{GPT-2}_{\text{\texttt{S}}}$ dominates the total runtime of DCD. However, as the number of denoising steps increases, this cost is amortized because, with KV-caching, the total computation cost for running $\text{GPT-2}_{\text{\texttt{S}}}$ stays constant.

\section{Additional Unconditional Generation Experiments}
\label{appx:additional-uncond-exps}

\begin{figure}[t]
    \centering
    \includegraphics[width=\linewidth]{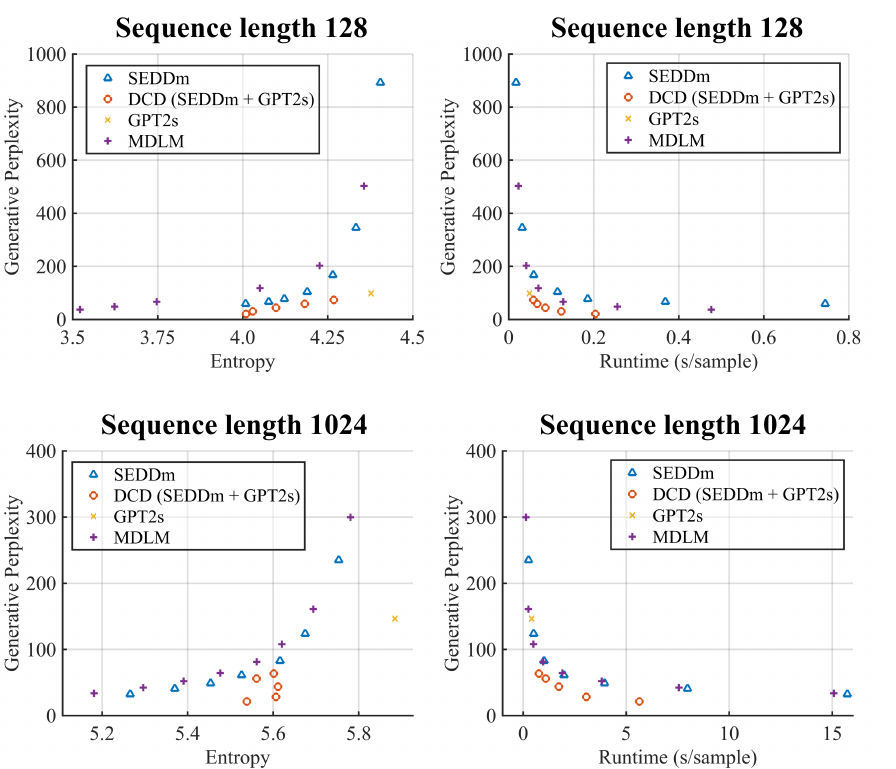}
    \vspace{-1.8em}
    \caption{Comparison between generative perplexity ($\downarrow$), diversity (measured by sentence entropy; $\uparrow$), and runtime ($\downarrow$) of DCD with baselines.}
    \label{fig:more-uncond-results}
\end{figure}

To better understand the relation between quality (measured by generative perplexity), diversity (measured by sentence entropy\footnote{The sentence entropy of a sequence is the entropy of its token frequency distribution. The reported number is averaged across all samples.}), and speed for DCD and its baselines. Specifically, we run the more efficient version of DCD described in the last paragraph of \cref{sec:overall-sampling} and \cref{appx:ar-iproj} to generate text sequences of lengths 128 and 1024. In addition to SEDD and GPT2, the two base models used by DCD, we compare them with MDLM \citep{sahoo2024simple}, a more recent discrete diffusion model that is more efficient than SEDD. Note that DCD can use any discrete diffusion model as its base model.

First, we compare the sample time and the generative perplexity (the second and the fourth sub-plot in \cref{fig:more-uncond-results}). Compared to SEDD, GPT, and MDLM, DCD consistently achieves better generative perplexity given a fixed runtime constraint. It also requires less time to achieve a desired perplexity value.

Additionally, we compare the perplexity and diversity of the generated text sequences. Following community standards, we adopt the sentence entropy to measure the diversity of generated text. Specifically, the entropy of each text sequence is the entropy of its token frequency distribution, and the final sentence entropy is the average entropy over all generated sequences. The desired behavior is to have low generative perplexity and high sentence entropy (which means high diversity). Results are shown in the table below and \cref{fig:more-uncond-results}'s first and third sub-plot. Compared to the two discrete diffusion models (SEDD and MDLM), DCD achieves better generative perplexity under the same entropy, which offers a better perplexity-diversity tradeoff. Compared to the autoregressive GPT model, although the entropy of DCD is lower, it achieves better generative perplexity with slightly worse entropy.

\section{Additional Experimental Details}
\label{appx:exp-details}

This section provides additional details of the experiments.

\subsection{Unconditional Text Generation}
\label{appx:exp-details-uncond-text}

\boldparagraph{SEDD.}
We adopt the SEDD-medium model with 320M non-embedding parameters trained on OpenWebText. The model is accessed through HuggingFace: \url{https://huggingface.co/louaaron/sedd-medium}. We follow the original paper \citep{lou2024discrete} and use the log-linear noise schedule $\sigma(t) \!=\! - \log (1 \!-\! (1 \!-\! \epsilon t))$, which leads to the forward transition probabilities ($0 \!\leq\! s \!<\! t \!\leq\! T$):
    \begin{align*}
        \q (\x_{t} \given \x_{s}) := \cat (\x_{t}; \exp (\sigma(t - s) \cdot Q) \cdot \x_{s}).
    \end{align*}
The absorbing mask forward noising process is used. The corresponding transition rate matrix is
    \begin{align*}
        Q := \begin{bmatrix}
            -1 & 0 & \cdots & 0 & 0 \\
            0 & -1 & \cdots & 0 & 0 \\
            \vdots & \vdots & \ddots & \vdots & \vdots \\
            0 & 0 & \cdots & -1 & 0 \\
            1 & 1 & \cdots & 1 & 0 \\
        \end{bmatrix},
    \end{align*}
\noindent where the last category is \texttt{<MASK>}.

\boldparagraph{GPT.}
The GPT-2-small model is obtained from HuggingFace: \url{https://huggingface.co/openai-community/gpt2}.

\boldparagraph{DCD.}
We implement DCD by combining $\text{SEDD}_{\text{\texttt{M}}}$ and $\text{GPT-2}_{\text{\texttt{S}}}$ following the steps in \cref{alg:sampling}. In line 8, instead of masking tokens independently, we group chunks of 8 tokens together and mask/unmask them with the same probability given the noise schedule (\ie $\alpha_{t} / \alpha_{t+1}$ as shown in Prop.~\ref{prop:mask-can-be-separated}).

\subsection{Conditional Text Generation}
\label{appx:exp-details-cond-text}

\boldparagraph{MAUVE Score.}
We adopt the MAUVE implementation available in the Python package \texttt{evaluate}. We use the default hyperparameters established by the original paper \citep{pillutla2021mauve}, which is also the default used by the package. We found that the number of samples and the number of samples given a fixed prompt influenced the score. Therefore, we randomly selected the 2,000 prompts and generated 5 samples for each prompt for all methods.

\boldparagraph{Detailed Runtime Analysis.}
As shown in \cref{alg:sampling}, in each denoising step of DCD, we need to run the discrete diffusion model twice: first to compute $\{\p (\tilde{X}_{t}^{i} \given \x_{t+1})\}_{i}$ and next to compute $\{\p (\tilde{X}_{t}^{i} \given \x_{t+1}^{<i})\}_{i}$ by applying causal attention masks to the same denoising neural network given that it is based on the Transformer architecture. Next, as discussed in \cref{appx:ar-iproj}, the total runtime consumed by the autoregressive model remains constant across different numbers of denoising steps thanks to the KV-caching mechanism. Therefore, the runtime of DCD will be dominated by the computation cost of the autoregressive model with only a few denoising steps. As the number of denoising steps increases, the runtime of the autoregressive model will be amortized and the total computation cost will be dominated by the cost to evaluate the diffusion model.

\boldparagraph{SSD-LM.}
SSD-LM \citep{han2023ssd} is a semi-autoregressive model that uses techniques from discrete diffusion models to predict/denoise chunks of sequences in an autoregressive manner. Specifically, given a predefined chunk size, SSD-LM diffuses tokens in each chunk one by one conditioned on all previous chunks. As a result, the model is semi-autoregressive and cannot see suffix prompts.

While the official implementation on GitHub (\url{https://github.com/xhan77/ssd-lm}) only allows conditioning on tokens in previous prompts, we improved their code to also allow conditioning on tokens in the current chunk that is being diffused. Specifically, we replace the diffusion model's input corresponding to the prompt tokens with the ground truth token embeddings.

We followed the original paper to choose a chunk size of 32 and use top-p sampling with $p \!=\! 0.95$. The remaining hyperparameters are kept as default.

\subsection{Antibody Sequence Infilling}
\label{appx:exp-details-protein-infill}

\boldparagraph{Detailed Task Description.}
The adopted antibodies with an immunoglobulin G (IgG) format, which comprises a heavy (H) chain and a light (L) chain. Each chain has three complementarity determining regions (CDRs) that are crucial toward the binding affinity to the target antigen.

\boldparagraph{Training NOS-D.}
We use the training script as well as the dataset provided in the official GitHub repo of NOS-D (\url{https://github.com/ngruver/NOS}). The model is trained with 50 epochs using the default settings (\eg learning rate and its schedule).

\boldparagraph{Training GPT.}
We use the same dataset provided in the repository of NOS-D and use the GPT implementation from \url{https://github.com/karpathy/nanoGPT/tree/master}. The GPT model has 6 layers, an embedding size of 512, and 16 attention heads. The model is trained for 10 epochs with the default settings in the nanoGPT repository.

\boldparagraph{DCD.}
When implementing DCD for the antibody sequence infilling task, we add an additional scaling factor to the coefficients in $\V$. That is, $\V$ is updated in line 6 of \cref{alg:sampling} following
    \begin{align*}
        \forall i, \tilde{x}_{t}^{i}, \, \V [i, \tilde{x}_{t}^{i}] = \beta \cdot \left ( \log \p_{\mathrm{dm}} (\tilde{x}_{t}^{i} \given \x_{t+1}) - \log \p_{\mathrm{dm}} (\tilde{x}_{t}^{i} \given \x_{t+1}^{<i}) \right ),
    \end{align*}
\noindent where we set $\beta \!=\! 0.1$ for this task. We note that $\beta \!=\! 1$ works well for the language modeling tasks. The need to choose a smaller $\beta$ in this task may be caused by the fact that the dataset and the models are much smaller and are more prone to overfitting.

\section{Additional Related Work}
\label{appx:additional-related-work}

We briefly review a class of related works that perform (semi-)autoregressive diffusion, which is weakly related to our work since we also ``combine'' discrete diffusion models with autoregressive models. Specifically, \citet{wu2023ar,chen2024diffusion} perform diffusion in an autoregressive manner by allowing the noising schedule to be variable-dependent. Variables at the beginning are kept unchanged at small $t$s and are corrupted only when $t$ is close to $T$; variables at the end are corrupted in the first few time steps. During sampling where $t$ moves from $T$ to $0$, initial variables/tokens are first denoised, followed by later tokens. This allows the diffusion model to perform ``autoregressive denoising''.

\section{Additional Text Samples}
\label{appx:additional-text-samples}

We provide randomly selected unconditional samples in \cref{fig:samples-dcd4-uncond,fig:samples-dcd32-uncond} and conditional samples in \cref{fig:samples-dcd4-cond,fig:samples-dcd32-cond}.

\begin{figure}[t]
    \centering
    \includegraphics[width=\columnwidth]{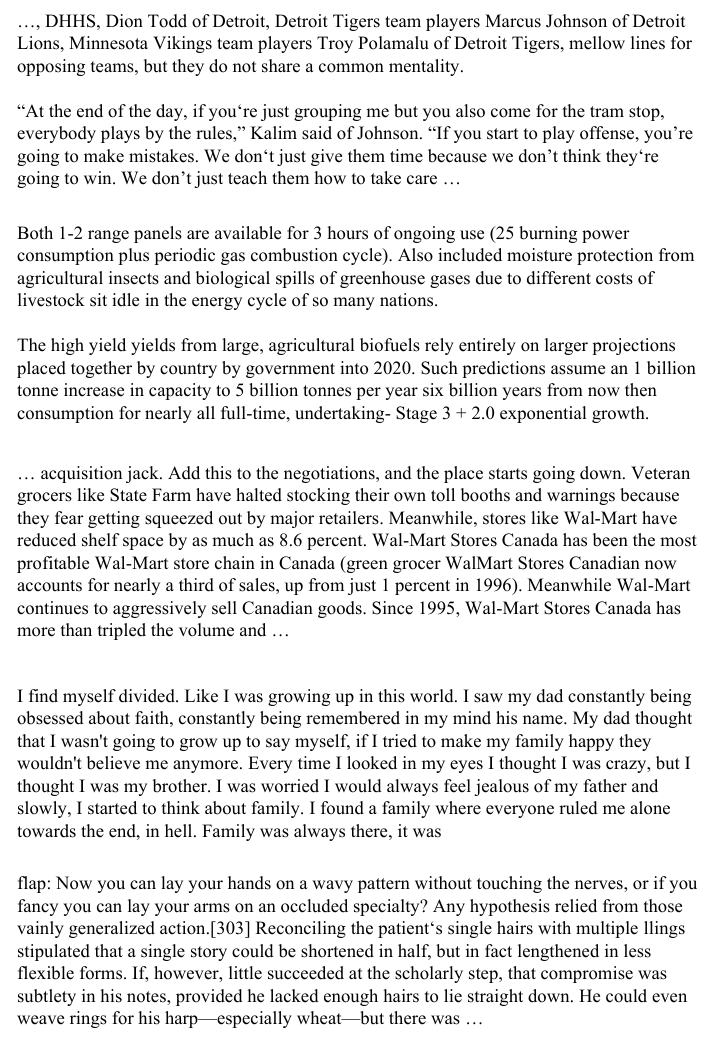}
    \caption{Randomly selected unconditional samples from DCD ($\text{SEDD}_{\text{\texttt{M}}} + \text{GPT-2}_{\text{\texttt{S}}}$) with 4 denoising steps.}
    \label{fig:samples-dcd4-uncond}
\end{figure}

\begin{figure}[t]
    \centering
    \includegraphics[width=\columnwidth]{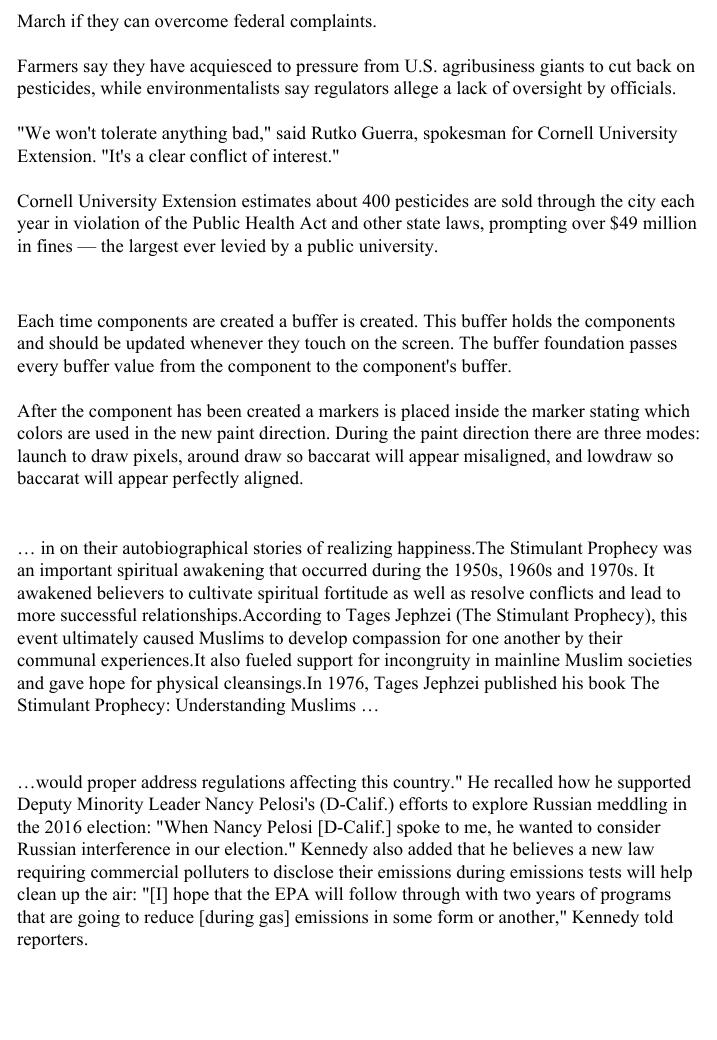}
    \caption{Randomly selected unconditional samples from DCD ($\text{SEDD}_{\text{\texttt{M}}} + \text{GPT-2}_{\text{\texttt{S}}}$) with 32 denoising steps.}
    \label{fig:samples-dcd32-uncond}
\end{figure}

\begin{figure}[t]
    \centering
    \includegraphics[width=\columnwidth]{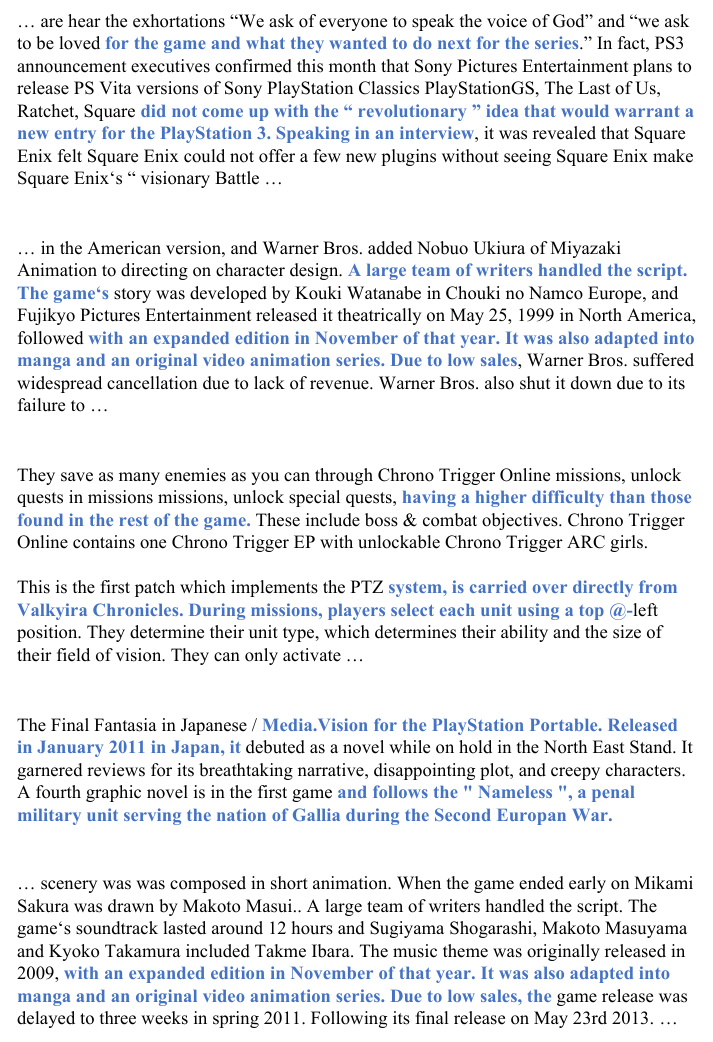}
    \caption{Randomly selected conditional samples from DCD ($\text{SEDD}_{\text{\texttt{M}}} + \text{GPT-2}_{\text{\texttt{S}}}$) with 4 denoising steps. Prompt texts are bolded and in blue.}
    \label{fig:samples-dcd4-cond}
\end{figure}

\begin{figure}[t]
    \centering
    \includegraphics[width=\columnwidth]{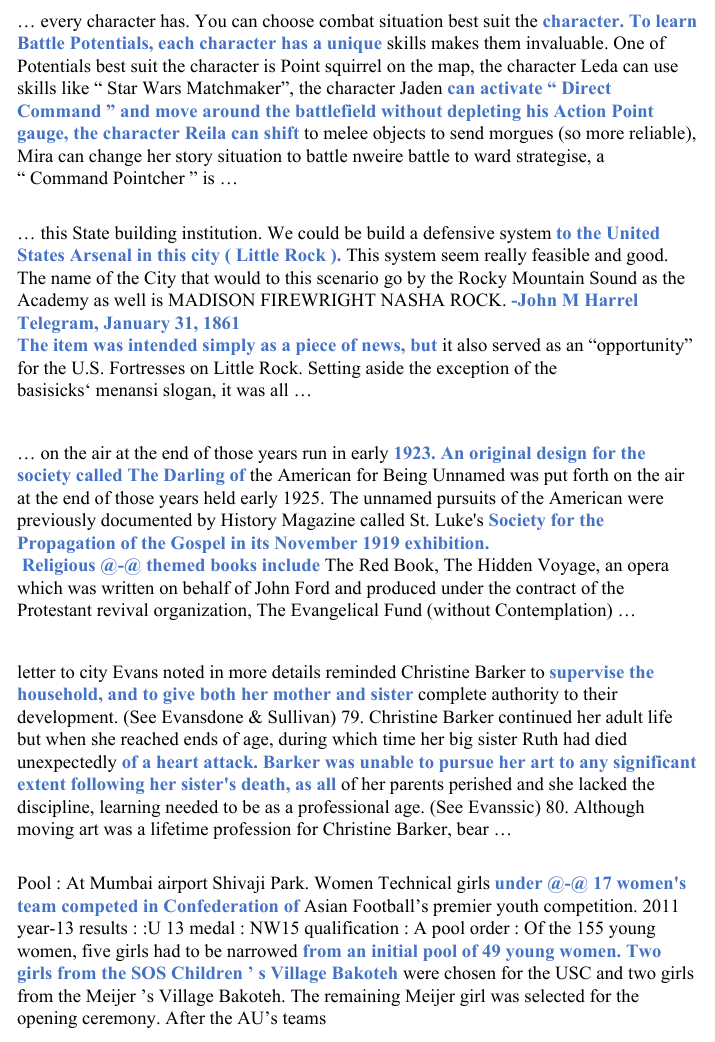}
    \caption{Randomly selected conditional samples from DCD ($\text{SEDD}_{\text{\texttt{M}}} + \text{GPT-2}_{\text{\texttt{S}}}$) with 32 denoising steps. Prompt texts are bolded and in blue.}
    \label{fig:samples-dcd32-cond}
\end{figure}

\end{document}